\documentclass[sigconf,nonacm]{acmart}

\usepackage{balance}
\usepackage{xcolor}

\renewcommand\footnotetextcopyrightpermission[1]{}

\title[IR-PRM+ with Previous-Utility Prediction]{IR-PRM+ with Previous-Utility Prediction: Constant Regret and Best-Iterate Convergence}

\author{Michael Sullins}
\affiliation{%
  \institution{University of Illinois Chicago}
  \city{Chicago}
  \state{Illinois}
  \country{USA}}
\email{sullins2@uic.edu}

\author{Ian A. Kash}
\affiliation{%
  \institution{University of Illinois Chicago}
  \city{Chicago}
  \state{Illinois}
  \country{USA}}
\email{iankash@uic.edu}

\begin{abstract}
Regret matching (and its variants) and counterfactual regret minimization are central algorithms in practice for normal-form and extensive-form game solving. Yet these have historically lacked the stronger theoretical convergence guarantees available to optimistic and extra-gradient methods. Recent work introduced IREG-PRM+, a regret matching variant which does have strong theoretical convergence guarantees. However, it relies on a simultaneous, extra-gradient update. IREG-PRM+ is an instance of the more general IR-PRM framework, where empirically optimistic and alternating updates often have better performance. However, the prior extra-gradient analysis does not extend to these cases.

We introduce a new approach that covers both settings. For simultaneous and alternating IR-PRM+ in finite two-player zero-sum matrix games, previous-utility prediction recovers the constant-regret and convergence-rate guarantees previously obtained with extra-gradient prediction: external regret bounded independently of $T$, an $O(T^{-1})$ average duality gap, and an $O(T^{-1/2})$ best-current-iterate gap. Moreover, the current duality gap vanishes and the current profile converges to the Nash equilibrium set. The key step bounds the component of consecutive strategy movement not controlled by the existing IR-PRM+ regret inequality using the relative growth of the nondecreasing regret-vector norm $\rho_t$. Coupling this bound across the two players yields recurrences whose effective coefficient scales as $1/(\rho_t^X\rho_t^Y)$, preventing unbounded norm growth. We extend the analysis to finitely many simplex learners under a weighted Minty condition and, through the corresponding CFR variant IR-PCFR+, recover analogous guarantees for finite two-player zero-sum perfect-recall extensive-form games under the same full-reach condition as prior work.
\end{abstract}


\newtheorem{theorem}{Theorem}
\newtheorem{lemma}{Lemma}
\newtheorem{fact}{Fact}
\newtheorem{corollary}{Corollary}
\newtheorem{remark}{Remark}
\newcommand{\Reg}{\operatorname{Reg}}
\newcommand{\Gap}{\operatorname{Gap}}
\newcommand{\R}{\mathbb{R}}
\newcommand{\1}{\mathbf{1}}
\newcommand{\norm}[1]{\left\lVert #1\right\rVert}
\newcommand{\ip}[2]{\left\langle #1,#2\right\rangle}
\newcommand{\pos}[1]{\left[#1\right]_+}

\begin{document}

\maketitle

\section{Introduction}
Regret minimization is a standard approach to equilibrium computation in games. Regret matching (RM), counterfactual regret minimization (CFR), and their variants are among the most successful algorithms for normal-form and extensive-form game solving. CFR+ combines RM+ with alternating updates and weighted averaging~\cite{tammelin2015}, while predictive RM+ (PRM+) uses optimism to obtain substantial empirical speedups within the same basic regret-matching framework~\cite{farina2021predictive}. This practical success has long coexisted with a theoretical gap: in
two-player zero-sum games, first-order methods such as Mirror-Prox,
the excessive-gap technique, and optimistic mirror descent admit
$O(T^{-1})$ ergodic convergence guarantees
~\cite{nemirovski2004,nesterov2005,rakhlin2013,chiang2012},
whereas standard RM+-based methods do not generally attain this rate;
in particular, PRM+ can exhibit $\Omega(T^{-1/2})$ convergence even in
self-play~\cite{farina2023instability,zhang2026}.

Zhang et al.~\cite{zhang2026} recently introduced the IR-PRM+ framework
to address this limitation. Its extra-gradient instantiation,
IREG-PRM+, was designed to retain the parameter-free, scale-invariant structure of
regret matching while maintaining a nondecreasing regret-vector norm and
satisfying an RVU-like~\cite{syrgkanis2015} bound.
IREG-PRM+ obtains constant regret in two-player zero-sum
self-play, the optimal $O(T^{-1})$ average-iterate rate, and an
$O(T^{-1/2})$ best-iterate guarantee, while remaining competitive with
PRM+ empirically~\cite{zhang2026}. Their analysis further connects
IREG-PRM+ to optimistic gradient descent with an adaptive learning rate,
providing a direct bridge between practical regret matching and the
stronger convergence theory of first-order methods.

The existing IR-PRM+ guarantees, however, do not cover two practically important settings. First, the constant-regret analysis of Zhang et al.~\cite{zhang2026} relies on extra-gradient prediction. Both players first form an intermediate profile $(\widetilde x_t,\widetilde y_t)$ by normalizing their regret vectors. The game is evaluated there to construct the current predictions, and the resulting played profile is then evaluated again, giving two evaluations per update. By contrast, previous-utility optimism forms the prediction directly from previously observed feedback and therefore requires only a single evaluation per update. This is the form of optimism used by PRM+~\cite{farina2021predictive}. Second, their fast-convergence theory does not cover alternating IR-PRM+ updates. Alternating updates are already a standard ingredient of CFR+, where they are associated with its strong practical performance~\cite{tammelin2015}. Subsequent work has likewise reported substantial empirical benefits from alternation, including strong performance of alternating PRM+ across several
benchmark games~\cite{cai2025last}. Zhang et al.~\cite{zhang2026} explicitly leave both previous-utility prediction and alternating variants as open theoretical questions. Their empirical results suggest that these previously unanalyzed settings are practically important: across all six reported benchmarks, when
performance is measured by gradient evaluations, the best-performing update is either alternating or simultaneous rather than extra-gradient~\cite{zhang2026}. 
% ...

The two settings pose different technical obstacles: previous-utility prediction requires control of a portion of strategy movement not controlled by the RVU bound used by Zhang et al.~\cite{zhang2026}, while alternation breaks the structure that lets the two players' RVU inequalities be combined directly. In greater detail, under previous-utility prediction,
\[
u_t^X-m_t^X=A(y_t-y_{t-1}),
\qquad
u_t^Y-m_t^Y=-A^\top(x_t-x_{t-1}),
\]
so the prediction error depends on the full movement between consecutive played strategies. The IR-PRM+ RVU bound, however, contains a negative term only for the prediction displacement $x_t-\widetilde x_t$, which captures only one part of that movement. Alternating updates introduce a separate difficulty: the second player responds to the first player's newly updated strategy, so the two players' updates are evaluated at different profiles. Previous-utility optimism therefore requires control of the missing part of consecutive strategy movement, while alternation additionally requires handling the staggered timing of the two updates.

Our main observation is that growth of the regret-vector norm itself controls the missing post-update displacement $\widetilde x_t-x_{t-1}$, a property not used by the extra-gradient argument. IR-PRM+ obtains its strategy by normalizing a nonnegative regret vector. If $\rho_t$ denotes its Euclidean norm, then a fixed-size change in the regret vector has less effect on the normalized strategy when $\rho_t$ is large. Quantitatively, the post-update displacement is controlled by the relative growth $\rho_t-\rho_{t-1}$ of the regret-vector norm. In self-play, the resulting feedback is doubly self-damping: a large value of a player's own $\rho_t$ limits how much a given prediction error can increase its norm, while a large opponent norm limits the opponent's post-update displacement and therefore reduces that prediction error. At the level of the coupled recurrence, the key inequality for the row player $X$ is, schematically,
\begin{equation}
\rho_{t+1}^X-\rho_t^X
\le
\frac{C}{\rho_t^X\rho_t^Y}
\left[
\rho_t^Y\norm{y_t-\widetilde y_t}_2^2
+(\rho_t^Y-\rho_{t-1}^Y)
\right],
\label{eq:intro-self-damping}
\end{equation}
with a symmetric inequality for the column player $Y$ and a time-independent, but problem-specific, constant $C$. The factor $1/\rho_t^X$ comes from converting the row player's prediction error into growth of its own regret-vector norm; the factor $1/\rho_t^Y$ comes from controlling the column player's post-update displacement. Hence the coefficient $1/(\rho_t^X\rho_t^Y)$ becomes smaller whenever either regret-vector norm grows. Once their product is sufficiently large, the positive prediction-error contribution is dominated by the negative prediction displacement term in the RVU inequality, preventing sustained further growth. Summing the two players' bounds and using the fact that their cumulative regrets equal a nonnegative multiple of the duality gap of the averaged profile then bounds the accumulated weighted squared prediction displacements by the regret-vector norms themselves. Substituting this bound back into the norm-growth inequality shows that both $\rho_t^X$ and $\rho_t^Y$ remain bounded for all $t$.

Our results show that, for simultaneous IR-PRM+, previous-utility prediction recovers the constant-regret and convergence-rate guarantees of IREG-PRM+ without the extra-gradient evaluation. In finite two-player zero-sum matrix games we obtain external regret bounded by constants independent of $T$, an $O(T^{-1})$ average duality gap, convergence of the current duality gap to zero, convergence of the current profile to the Nash equilibrium set, and an $O(T^{-1/2})$ best-current-iterate guarantee. We then extend the analysis to alternating updates. In this setting, the two players' prediction errors arise at different points within each full update because one player updates before the other. The proof therefore works with the intermediate profile after the first player's update and shows that the corresponding bounds combine into regret and current-profile guarantees once both players have updated. Thus the same regret-vector norm-growth argument extends to both previous-utility prediction and alternating updates. It also yields a two-displacement RVU bound for IR-PRM+ that controls both the prediction displacement $x_t-\widetilde x_t$ and the post-update displacement $\widetilde x_{t+1}-x_t$, extending the previous RVU form with the additional control needed for previous-utility prediction. Although our convergence analysis does not rely on this refined RVU bound, it shows that the same theory recovers the missing form directly.

Finally, we extend the argument from the two players of a matrix game to the many local learners used by CFR. A counterfactual utility at one information set can depend on strategies at many other information sets, so the two-player recurrence must be replaced by a product-of-simplices coupling. We show that, under the weighted Minty condition of Zhang et al.~\cite{zhang2026}, previous-utility prediction keeps all local regret-vector norms bounded. In their CFR setting, the full-reach assumption ensures that the weights in this condition are strictly positive. Our boundedness result can therefore be combined with their IR-PCFR+ reduction to obtain analogous guarantees for finite two-player zero-sum perfect-recall extensive-form games. The result applies to both simultaneous and alternating IR-PCFR+: bounded sequence-form external regret, an $O(T^{-1})$ average duality gap, convergence of the current sequence-form gap to zero and of the current realization-plan profile to the sequence-form Nash set, and an $O(T^{-1/2})$ best-current-iterate guarantee.

\section{Related Work}
CFR decomposes global external regret into local counterfactual regrets and thereby makes regret matching practical in large extensive-form games~\cite{zinkevich2007}. CFR+ replaces regret matching with RM+ and combines it with alternating updates and weighted averaging, producing substantial empirical improvements while retaining the usual $O(T^{-1/2})$ worst-case convergence rate~\cite{tammelin2015}. Farina et al.~\cite{farina2021predictive} subsequently developed predictive RM and RM+ through predictive Blackwell approachability and showed that PRM+ can substantially accelerate CFR in practice. These works established predictive regret matching as a practical route to faster CFR convergence, while retaining the classical regret-matching framework.

A central framework for fast convergence through stability and predictability is the Regret bounded by Variation in Utilities (RVU) property, introduced by Syrgkanis et al.~\cite{syrgkanis2015}. An RVU bound controls regret through variation in the utility sequence together with a negative penalty on the learner's own movement; Syrgkanis et al. showed that recency-biased regularized learners satisfy such bounds and that, when these learners interact, the variation and movement terms can be combined to obtain faster-than-worst-case convergence in games. 
Farina et al.~\cite{farina2023instability} showed that RM+ and predictive
RM+ can be unstable and obtained faster convergence through stabilized
variants, including Smooth PRM+, which require parameter tuning. Cai et al.~\cite{cai2025last} studied last-iterate convergence directly: they exhibit nonconvergent behavior for several standard RM+-based dynamics and prove last-iterate guarantees for extra-gradient RM+ and smooth predictive RM+. Meng et al.~\cite{meng2025smooth} established last-iterate and finite-time best-iterate convergence for a smooth RM+ variant under the weak Minty condition. More broadly, Cai et al.~\cite{cai2025reduction} gave a black-box reduction from average-iterate to last-iterate convergence for bilinear games, while Meng et al.~\cite{meng2026faster} introduced a different parameter-free smooth predictive RM+ variant with an $O(T^{-1})$ average-iterate convergence guarantee. In extensive-form games, Lee et al.~\cite{lee2021last} established last-iterate convergence for optimistic regret-minimization methods operating directly on treeplexes. Chakrabarti et al.~\cite{chakrabarti2024treeplex} developed Blackwell-approachability methods on treeplexes, including a stabilized variant with an $O(T^{-1})$ convergence rate, while Meng et al.~\cite{meng2025efg} obtained last-iterate convergence through an RM-based CFR method for a sequence of perturbed regularized extensive-form games.
Taken together, these works obtain stronger convergence guarantees through
several modifications of regret-matching dynamics and related first-order
methods. Our results show that, within IR-PRM+, previous-utility prediction
and alternating updates retain constant regret, the
$\mathcal O_T(T^{-1})$ average-gap rate, the
$\mathcal O_T(T^{-1/2})$ best-current-iterate rate, and current-profile
convergence without extra-gradient prediction.

\section{Preliminaries and IR-PRM+}
\label{sec:setup}
% Let $\Delta_k:=\{p\in\R_+^k:\ip{p}{\1}=1\}$ denote the $k$-action simplex, and let $A\in\R^{m\times n}$ be the payoff matrix of a two-player zero-sum game. At round $t$, the row player chooses $x_t\in\Delta_m$ and the column player chooses $y_t\in\Delta_n$; the row player receives expected payoff $x_t^\top Ay_t$ and the column player its negative. The utilities of each player's pure actions against the opponent's current mixed strategy are
Let $\Delta_k:=\{p\in\R_+^k:\ip{p}{\1}=1\}$ denote the $k$-action simplex, and let $A\in\R^{k^X\times k^Y}$ be the payoff matrix of a two-player zero-sum game. At round $t$, the row player chooses $x_t\in\Delta_{k^X}$ and the column player chooses $y_t\in\Delta_{k^Y}$; the row player receives expected payoff $x_t^\top Ay_t$ and the column player its negative. The utilities of each player's pure actions against the opponent's current mixed strategy are
\[
u_t^X=Ay_t,
\qquad
u_t^Y=-A^\top x_t.
\]
 We distinguish two update schedules. Under \emph{simultaneous} updates, both players choose their round-$t$ strategies from information available before the joint play and then update after observing the round-$t$ utilities. Under \emph{alternating} updates, the players update sequentially, so the second update can use the first player's newly computed strategy. The analysis in this paper is simultaneous unless stated otherwise. For later use, let
$\bar x_T=T^{-1}\sum_{t=1}^T x_t$ and $\bar y_T=T^{-1}\sum_{t=1}^T y_t$ denote the average strategies. The duality gap is
\[
% \Gap(x,y)
% :=
% \max_{x'\in\Delta_m}(x')^\top Ay
% -
% \min_{y'\in\Delta_n}x^\top Ay'.
\Gap(x,y)
:=
\max_{x'\in\Delta_{k^X}}(x')^\top Ay
-
\min_{y'\in\Delta_{k^Y}}x^\top Ay'.
\]

Suppressing player superscripts, consider a generic $k$-action learner. It chooses $x_t\in\Delta_k$, observes a utility vector $u_t\in\R^k$, and incurs external regret
\[
\Reg(T)=\max_{p\in\Delta_k}\sum_{t=1}^T\ip{u_t}{p-x_t}.
\]
We now describe one IR-PRM+ learner~\cite{zhang2026}. 
Each round the learner begins with a nonnegative cumulative regret vector
$\tilde r_t\in\R_+^k$ produced at the end of the preceding round, and we let
$\rho_t:=\norm{\tilde r_t}_2$ denote its Euclidean norm. 
For a $k$-action learner, we use the symmetric unit-norm initialization, using $\1$ to denote the all-ones
vector
\begin{equation}
\tilde r_1=\frac{1}{\sqrt{k}}\1,
\qquad
\rho_1:=\norm{\tilde r_1}_2=1.
\label{eq:positive-init}
\end{equation}
The initialization in \eqref{eq:positive-init} is used only to avoid
division by zero in the main normalization formulas. IR-PRM+ permits any
nonnegative initial regret vector, and Appendix~F extends
all results to the standard zero initialization $\tilde r_1=0$;
only rounds with $\tilde r_t=0$ and the first transition to $\rho_t>0$
require separate treatment.

The cumulative regret vector induces
\[
\tilde x_t=\frac{\tilde r_t}{\norm{\tilde r_t}_1}.
\]
Before choosing its strategy, the learner receives a prediction $m_t\in\R^k$
of the utility vector $u_t$. The prediction step adds $m_t$ to the regret
vector and applies a common scalar shift $\gamma_t$, chosen uniquely so that
the positive part preserves its Euclidean norm:
\[
\norm{\pos{\tilde r_t+m_t-\gamma_t\1}}_2=\rho_t.
\]
Here $\pos{z}$ denotes the coordinatewise positive part. Set
\[
r_t=\tilde r_t+m_t-\gamma_t\1,
\qquad
x_t=\frac{\pos{r_t}}{\norm{\pos{r_t}}_1}.
\]
The distinction between $\tilde x_t$ and $x_t$ will matter later:
$\tilde x_t$ is the strategy induced by $\tilde r_t$ before prediction,
whereas $x_t$ is the strategy actually played after incorporating $m_t$.
After observing $u_t$, define the prediction error
\begin{equation}
g_t=(u_t-m_t)-\ip{u_t-m_t}{x_t}\1,
\label{eq:gdef}
\end{equation}
and update
\begin{equation}
\tilde r_{t+1}=\pos{r_t+g_t}.
\label{eq:update}
\end{equation}
% The regret-vector norm is nondecreasing (see Appendix~A),
% so \eqref{eq:positive-init} implies
Zhang et al.~\cite{zhang2026} establish that the regret-vector norm is
nondecreasing, so \eqref{eq:positive-init} implies
\begin{equation}
\rho_t\ge1\qquad\text{for every }t.
\label{eq:rho-lower-one}
\end{equation}

For their constant-regret analysis, Zhang et al.~\cite{zhang2026} require an extra-gradient evaluation before the round-$t$ play. Starting from the current regret vectors, IR-PRM+ first gives the strategies $\tilde x_t$ and $\tilde y_t$. Zhang et al. then evaluate the game at this fresh pair and use the resulting utility vectors as the predictions
\[
m_t^X=A\tilde y_t,
\qquad
m_t^Y=-A^\top\tilde x_t.
\]
These predictions are fed into the predictive step above to produce the strategies $x_t$ and $y_t$ that are actually played. The game is then evaluated again at $(x_t,y_t)$ to obtain the round-$t$ utilities used in the update. Thus extra-gradient prediction requires a preliminary game evaluation at $(\tilde x_t,\tilde y_t)$ before the usual evaluation at the played strategies.

With previous-utility prediction, there is no preliminary game evaluation. We set $m_1^X=m_1^Y=0$, and from round $t=2$ onward simply reuse the utility observed on the preceding round:
\[
m_t^X=Ay_{t-1},
\qquad
m_t^Y=-A^\top x_{t-1}.
\]
IR-PRM+ then uses this prediction to produce $x_t$ and $y_t$ and evaluates the game only at the played pair.

\paragraph{Asymptotic notation.}
Throughout, $\mathcal O_T(\cdot)$ suppresses constants independent of $T$, including dependence on action dimensions, weights in the weighted Minty condition, initialization, and the game.

%; in Appendix~F, they may also depend on the fixed first positive regret-vector norms. 

\section{Starting from the IR-PRM+ RVU Bound}
\label{sec:rvu}

The constant-regret analysis of Zhang et al.~\cite{zhang2026} is based on an RVU-type inequality for IR-PRM+. We use the following intermediate form from their proof, which keeps the terminal regret-vector norm explicit.

\begin{fact}[Intermediate inequality from Zhang et al.~\cite{zhang2026}]
\label{fact:rvu-intermediate}
For IR-PRM+ with any nonnegative initial regret vector,
\begin{equation}
\Reg(T)
\le
\rho_{T+1}
-
\sum_{t=1}^T
\frac{\rho_t}{2k}
\norm{x_t-\tilde x_t}_2^2.
\label{eq:rvu-intermediate}
\end{equation}
\end{fact}

The negative term in~\eqref{eq:rvu-intermediate} controls the prediction displacement from the intermediate strategy $\tilde x_t$ to the played strategy $x_t$. Under extra-gradient prediction, this is directly aligned with the prediction error: for the two players, the corresponding errors are $A(y_t-\tilde y_t)$ and $-A^\top(x_t-\tilde x_t)$. Previous-utility prediction changes this relationship. For $t\ge2$,
\begin{align}
u_t^X-m_t^X
&=A(y_t-y_{t-1}),
\label{eq:prev-error-X}\\
u_t^Y-m_t^Y
&=-A^\top(x_t-x_{t-1}).
\label{eq:prev-error-Y}
\end{align}
Thus the prediction error depends on the full movement between consecutive played strategies, whereas~\eqref{eq:rvu-intermediate} directly controls only the prediction displacement from $\tilde x_t$ to $x_t$. The remaining part is the post-update displacement from $x_{t-1}$ to $\tilde x_t$. The next section relates this post-update displacement to growth of the same regret-vector norm $\rho_t$ that appears in~\eqref{eq:rvu-intermediate}.

Proofs for the results in Sections~\ref{sec:rvu}--\ref{sec:coupling} that do not appear in the main text are collected in Appendix~A.

\section{Single-Learner Displacement and Movement Bounds}
\label{sec:geometry}

Consider a single generic IR-PRM+ learner. The negative term in Eq.~\eqref{eq:rvu-intermediate} controls the prediction displacement from $\tilde x_t$ to the played strategy $x_t$, but previous-utility prediction depends on the full movement between consecutive played strategies. The missing piece is therefore the  post-update displacement from $x_t$ to $\tilde x_{t+1}$. The key observation is that this post-update displacement is controlled by the increase in the regret-vector norm. Lemma~\ref{lem:update-change} makes this relation quantitative.

\begin{lemma}[Post-update displacement bound]
\label{lem:update-change}
For a $k$-action IR-PRM+ learner and every $t\ge1$,
\begin{equation}
\norm{\tilde x_{t+1}-x_t}_2^2
\le
\frac{2(1+\sqrt{k})^2(\rho_{t+1}-\rho_t)}{\rho_{t+1}}.
\label{eq:update-change}
\end{equation}
\end{lemma}

Zhang et al.~\cite{zhang2026} establish that the IR-PRM+ regret-vector norm is nondecreasing, $\rho_{t+1}\ge\rho_t$. Their update also gives
\[
\norm{\pos{r_t}}_2=\rho_t,
\qquad
\norm{\tilde r_{t+1}}_2=\rho_{t+1},
\]
and their monotonicity argument yields
\begin{equation}
\ip{\tilde r_{t+1}}{\pos{r_t}}
\ge
\rho_t^2.
\label{eq:inner-lower}
\end{equation}
We use these properties to relate the increase in the regret-vector norm to the corresponding post-update displacement.

\begin{proof}
Since $x_t$ and $\tilde x_{t+1}$ are the $\ell_1$ normalizations of
$\pos{r_t}/\rho_t$ and $\tilde r_{t+1}/\rho_{t+1}$, respectively, an elementary normalization inequality for nonnegative Euclidean unit vectors gives
\begin{equation}
\norm{\tilde x_{t+1}-x_t}_2^2
\le
(1+\sqrt{k})^2
\norm{
\frac{\tilde r_{t+1}}{\rho_{t+1}}
-
\frac{\pos{r_t}}{\rho_t}
}_2^2.
\label{eq:normalized-strategy}
\end{equation}
See the normalization bound in Appendix~G.

As we have Euclidean unit vectors, 
\begin{align}
\norm{
\frac{\tilde r_{t+1}}{\rho_{t+1}}
-
\frac{\pos{r_t}}{\rho_t}
}_2^2
&=
\norm{\frac{\tilde r_{t+1}}{\rho_{t+1}}}_2^2
+
\norm{\frac{\pos{r_t}}{\rho_t}}_2^2
-
2\ip{
\frac{\tilde r_{t+1}}{\rho_{t+1}}
}{
\frac{\pos{r_t}}{\rho_t}
}
\nonumber\\
&=
1+1-
\frac{2\ip{\tilde r_{t+1}}{\pos{r_t}}}
{\rho_t\rho_{t+1}}
\nonumber\\
&\le
2-\frac{2\rho_t^2}{\rho_t\rho_{t+1}}
\nonumber\\
% &=
% 2-\frac{2\rho_t}{\rho_{t+1}}
% \nonumber\\
% &=
% \frac{2\rho_{t+1}}{\rho_{t+1}}
% -
% \frac{2\rho_t}{\rho_{t+1}}
% \nonumber\\
&=
\frac{2(\rho_{t+1}-\rho_t)}{\rho_{t+1}}.
\label{eq:normalized-change}
\end{align}
The inequality uses~\eqref{eq:inner-lower}. Substituting
\eqref{eq:normalized-change} into~\eqref{eq:normalized-strategy}
gives~\eqref{eq:update-change}.
\end{proof}

Together with the prediction displacement already controlled by~\eqref{eq:rvu-intermediate}, Lemma~\ref{lem:update-change} gives the two pieces needed to bound movement between consecutive played strategies. This also leads to a two-displacement RVU bound with explicit terms for $x_t-\tilde x_t$ and $\tilde x_{t+1}-x_t$; see the two-displacement RVU bound in Appendix~H. This strengthened RVU bound is not used in the main convergence analysis, which works directly with regret-vector norm growth.

% Lemma~\ref{lem:update-change} therefore goes beyond the known monotonicity of the IR-PRM+ regret-vector norm by quantifying how its increase controls the post-update displacement from $x_t$ to $\tilde x_{t+1}$. Together with the prediction displacement from $\tilde x_t$ to $x_t$ already controlled by~\eqref{eq:rvu-intermediate}, this gives the two pieces needed to bound the full movement between consecutive played strategies.

% Combining this post-update displacement bound with the inherited RVU
% inequality also yields a two-displacement RVU bound for IR-PRM+, with
% explicit terms for both $x_t-\tilde x_t$ and
% $\tilde x_{t+1}-x_t$; see Appendix~H. This strengthened
% single-learner inequality does not by itself establish current-profile
% convergence or handle alternating updates: both require controlling the
% coupled growth of the players' regret-vector norms, and alternation
% additionally requires a staggered analysis. Our main argument therefore
% works directly with regret-vector norm growth.

%We retain explicit constants in the structural movement and norm-growth inequalities whenever they are needed to close a threshold argument.

% Thus prediction error depends on the full movement between consecutive played strategies, whereas the inherited stability term controls only $x_t-\tilde x_t$ or $y_t-\tilde y_t$. Under extra-gradient prediction, the corresponding errors 

\begin{corollary}[Consecutive-iterate movement]
\label{cor:movement}
For a $k$-action learner and every $t\ge1$,
\begin{equation}
\norm{x_{t+1}-x_t}_2^2
\le
2\norm{x_{t+1}-\tilde x_{t+1}}_2^2
+
\frac{4(1+\sqrt{k})^2(\rho_{t+1}-\rho_t)}{\rho_{t+1}}.
\label{eq:movement}
\end{equation}
\end{corollary}

\begin{proof}
Insert $\tilde x_{t+1}$ between $x_t$ and $x_{t+1}$ and use
$\norm{a+b}_2^2\le2\norm a_2^2+2\norm b_2^2$:
\[
\norm{x_{t+1}-x_t}_2^2
\le2\norm{x_{t+1}-\tilde x_{t+1}}_2^2
+2\norm{\tilde x_{t+1}-x_t}_2^2.
\]
Lemma~\ref{lem:update-change} bounds the last term and gives
\eqref{eq:movement}.
\end{proof}

\section{Two-Player Analysis}
\label{sec:coupling}
We now return to the two-player game. Our goal is to control the growth of the two regret-vector norms. Under previous-utility prediction, each player's prediction error is driven by the opponent's consecutive strategy movement; Corollary~\ref{cor:movement} expresses that movement in terms of prediction displacement and the opponent's norm growth, which controls the post-update displacement. The proof below combines this movement bound with a one-step IR-PRM+ norm-growth estimate to obtain the coupled recurrences. Here and below, $\|A\|_2$ denotes the spectral norm of the game matrix.

\begin{lemma}[Regret-vector norm growth in self-play]
\label{lem:norm-growth}
Let $\kappa_A:=2(1+\sqrt{k^X})^2(1+\sqrt{k^Y})^2\|A\|_2^2$.
For every $t\ge2$,
\begin{align}
\rho_{t+1}^X-\rho_t^X
&\le \frac{\kappa_A}{\rho_t^X\rho_t^Y}
\left[\frac{\rho_t^Y}{2k^Y}\norm{y_t-\tilde y_t}_2^2
+\rho_t^Y-\rho_{t-1}^Y\right],
\label{eq:normX}\\
\rho_{t+1}^Y-\rho_t^Y
&\le \frac{\kappa_A}{\rho_t^X\rho_t^Y}
\left[\frac{\rho_t^X}{2k^X}\norm{x_t-\tilde x_t}_2^2
+\rho_t^X-\rho_{t-1}^X\right].
\label{eq:normY}
\end{align}
\end{lemma}

\begin{proof}
We first record a norm-growth result used below. Zhang et al.~\cite{zhang2026} give
\[
\rho_{t+1}^2-\rho_t^2\le\norm{g_t}_2^2,
\]

and their monotonicity result gives $\rho_{t+1}\ge\rho_t$. Therefore,
\begin{equation}
0\le
\rho_{t+1}-\rho_t
\le
\frac{\norm{g_t}_2^2}{\rho_{t+1}+\rho_t}
\le
\frac{\norm{g_t}_2^2}{2\rho_t}.
\label{eq:growth}
\end{equation}

For $v\in\R^h$ and $x\in\Delta_h$, the
Cauchy--Schwarz inequality, together with $\norm x_2\le1$, gives
\begin{equation}
\norm{v-\ip{v}{x}\1}_2\le(1+\sqrt{h})\norm v_2.
\label{eq:center-bound}
\end{equation}
For the row player, previous-utility prediction yields
$u_t^X-m_t^X=A(y_t-y_{t-1})$. Applying \eqref{eq:center-bound} and 
$\norm{Az}_2\le\norm A_2\norm z_2$ to 
\eqref{eq:growth}, we obtain
\begin{equation}
\rho_{t+1}^X-\rho_t^X
\le
\frac{(1+\sqrt{k^X})^2\norm A_2^2}{2\rho_t^X}
\norm{y_t-y_{t-1}}_2^2.
\label{eq:normX-before-movement}
\end{equation}
Corollary~\ref{cor:movement} for the column player at time $t-1$ gives
\[
\norm{y_t-y_{t-1}}_2^2
\le2\norm{y_t-\tilde y_t}_2^2
% +\frac{4(1+\sqrt n)^2}{\rho_t^Y}
+\frac{4(1+\sqrt{k^Y})^2}{\rho_t^Y}
(\rho_t^Y-\rho_{t-1}^Y).
\]
Substituting into \eqref{eq:normX-before-movement} and using the
specified value of $\kappa_A$ yields
\begin{align*}
\rho_{t+1}^X-\rho_t^X
&\le\frac{\kappa_A}{\rho_t^X\rho_t^Y}
% \biggl[\frac{\rho_t^Y}{2(1+\sqrt n)^2}
\biggl[\frac{\rho_t^Y}{2(1+\sqrt{k^Y})^2}
\norm{y_t-\tilde y_t}_2^2\\[-0.3em]
&\hspace{4.9em}+\rho_t^Y-\rho_{t-1}^Y\biggr].
\end{align*}
% Because $(1+\sqrt n)^2\ge n$, this implies \eqref{eq:normX}.
Because $(1+\sqrt{k^Y})^2\ge k^Y$, this implies \eqref{eq:normX}.
Exchanging the players and using $\norm{A^\top}_2=\norm A_2$
proves \eqref{eq:normY}.
\end{proof}

The coupled coefficient $1/(\rho_t^X\rho_t^Y)$ is the self-damping term used in the argument below.
% Here the $\rho_t^X$ comes from \eqref{eq:growth}; the second
% comes from the opponent's movement bound.

% New and improved
\begin{theorem}[Constant regret]
\label{thm:main}
In any finite two-player zero-sum matrix game, simultaneous IR-PRM+ with
previous-utility prediction has external regrets bounded by a finite game-dependent constant
independent of $T$. Consequently,
\[
\Gap(\bar x_T,\bar y_T)=\mathcal O_T(T^{-1}).
\]
\end{theorem}

%\paragraph{Proof sketch.}
\begin{proof}
By \eqref{eq:rho-lower-one}, both regret-vector norms are at least one
from the outset. Since the norms are also nondecreasing, so is their
product. Consider the sufficient threshold
\begin{equation}
% \rho_t^X\rho_t^Y
% \ge
% 8(1+\sqrt m)^2(1+\sqrt n)^2\|A\|_2^2.
\rho_t^X\rho_t^Y
\ge
8(1+\sqrt{k^X})^2(1+\sqrt{k^Y})^2\|A\|_2^2.
\label{eq:threshold}
\end{equation}
If this threshold is never reached, then, since
$\rho_t^X,\rho_t^Y\ge1$, each norm is bounded above by their product and
hence by the right-hand side of~\eqref{eq:threshold}. If the threshold is
reached, monotonicity ensures that it continues to hold. Therefore, at all times after the threshold is reached the coefficient in \eqref{eq:normX} and \eqref{eq:normY} is at most
\[
\frac{\kappa_A}{\rho_t^X\rho_t^Y}
\le
% \frac{
% 2(1+\sqrt m)^2(1+\sqrt n)^2\|A\|_2^2
% }{
% 8(1+\sqrt m)^2(1+\sqrt n)^2\|A\|_2^2
% }
\frac{
2(1+\sqrt{k^X})^2(1+\sqrt{k^Y})^2\|A\|_2^2
}{
8(1+\sqrt{k^X})^2(1+\sqrt{k^Y})^2\|A\|_2^2
}
=
\frac14.
\]

% To see why this is enough, define
% \[
% S_t:=\rho_t^X+\rho_t^Y
% \]
% and
% \[
% H_t:=
% \frac{\rho_t^X}{2k^X}\norm{x_t-\tilde x_t}_2^2
% +
% \frac{\rho_t^Y}{2k^Y}\norm{y_t-\tilde y_t}_2^2.
% \]
To see why this is enough, define $S_t:=\rho_t^X+\rho_t^Y$ and
\[
H_t:=
\frac{\rho_t^X}{2k^X}\norm{x_t-\tilde x_t}_2^2
+
\frac{\rho_t^Y}{2k^Y}\norm{y_t-\tilde y_t}_2^2.
\]

Adding \eqref{eq:normX} and \eqref{eq:normY} then gives,
after the threshold is reached,
\[
S_{t+1}-S_t
\le
\frac14 H_t+\frac14(S_t-S_{t-1}).
\]
Let $\tau$ denote the first round from which the threshold holds. Summing
from $\tau$ to $T$ gives
\[
S_{T+1}-S_\tau
\le
\frac14\sum_{t=\tau}^T H_t
+
\frac14(S_T-S_{\tau-1}).
\]
Since $S_T\le S_{T+1}$, moving the resulting quarter of
$S_{T+1}-S_\tau$ to the left and absorbing the finite threshold-crossing
terms into a constant gives
\begin{equation}
S_{T+1}
\le
C_0+\frac13\sum_{t=1}^T H_t,
\label{eq:preceding-norm-growth-bound}
\end{equation}
where $C_0$ is finite and independent of $T$.

The remaining step is to control the accumulated weighted squared
prediction displacement terms. Applying
Fact~\ref{fact:rvu-intermediate} to each player and using the fact that the game is zero-sum gives
\begin{equation}
0
\le
T\,\Gap(\bar x_T,\bar y_T)
=
\Reg_X(T)+\Reg_Y(T)
\le
S_{T+1}-\sum_{t=1}^T H_t.
\label{eq:zero-sum-identity}
\end{equation}
Hence $\sum_{t=1}^T H_t\le S_{T+1}$. Substituting this into
\eqref{eq:preceding-norm-growth-bound} yields
$S_{T+1}\le C_0+\frac13S_{T+1}$, and therefore
$S_{T+1}\le\frac32C_0$.
\end{proof}

% Theorem~\ref{thm:main} provides the desired linear rate for the average profiles.
% The same bounds can also be used to control the actual simultaneous profiles. The additional ingredient is a residual lemma for the predicted utility (see Appendix~A). For a utility vector $v$ and strategy $x\in\Delta_k$, the quantity $\max_{p\in\Delta_k}\ip{v}{p-x}$ is the improvement available from switching to a best response. When the regret-vector norm is bounded, this residual for the predicted utility is controlled by the prediction-step displacement $\norm{x_t-\tilde x_t}_2$. Intuitively, if incorporating the prediction barely changes the normalized regret vector, then the prediction cannot offer a large profitable deviation. Under previous-utility prediction, Corollary~\ref{cor:movement} also makes the predicted and current utilities asymptotically close by controlling consecutive strategy movement. Combining these two facts transfers the square-summable movement estimates underlying the proof of Theorem~\ref{thm:main} to the current duality gaps.

Theorem~\ref{thm:main} gives the desired
$\mathcal O_T(T^{-1})$ average-gap rate.
The same bounds also control the current simultaneous profiles. The additional
ingredient is a predicted-utility residual bound (Lemma 4 in Appendix~A.2). For a utility vector $v$ and strategy $x\in\Delta_k$, the quantity
$\max_{p\in\Delta_k}\ip{v}{p-x}$ is the improvement available from switching
to a best response. When the regret-vector norm is bounded, this residual for
the predicted utility is controlled by the prediction displacement
$\norm{x_t-\tilde x_t}_2$. Intuitively, if incorporating the prediction barely
changes the normalized regret vector, then the prediction cannot cause a large
profitable deviation. Under previous-utility prediction,
Corollary~\ref{cor:movement} also makes the predicted and current utilities
asymptotically close by controlling consecutive strategy movement. Combining
these facts with the square-summable strategy-movement bounds underlying
Theorem~\ref{thm:main} yields square summability of the current duality gaps.

Let $\mathcal Z^\star$ denote the Nash equilibrium set. Here and below, $\operatorname{dist}$ denotes Euclidean distance to a set.
\begin{theorem}[Current-profile convergence]
\label{thm:current-matrix}
In the same setting,
\begin{equation}
\sum_{t=1}^{\infty}\Gap(x_t,y_t)^2<\infty.
\label{eq:matrix-gap-square-sum}
\end{equation}
Consequently,
\[
\Gap(x_t,y_t)\to0,
\qquad
\operatorname{dist}\bigl((x_t,y_t),\mathcal Z^\star\bigr)\to0.
\]
Moreover,
\begin{equation}
\min_{1\le t\le T}\Gap(x_t,y_t)=\mathcal O_T(T^{-1/2}).
\label{eq:matrix-best-current}
\end{equation}
\end{theorem}

% DO NOT DELETE (just in case)
% Appendix~I develops several consequences of Theorem~\ref{thm:two-displacement-rvu}, including a bound on consecutive strategy movement, a prediction-error form of the RVU inequality, and its application to simultaneous two-player zero-sum self-play.

\section{Simultaneous IR-PCFR+ in Extensive-Form Games}
\label{sec:cfr}
The CFR reduction follows Zhang et al.~\cite{zhang2026}, who use the standard CFR decomposition to show that a full-reach Nash equilibrium induces a weighted Minty condition for the counterfactual-utility operator. We retain this reduction but establish the convergence guarantees under previous-utility prediction rather than extra-gradient prediction. This requires extending our two-player norm-growth analysis to finitely many coupled simplex learners. Combined with the existing CFR reduction, the resulting bounds yield the simultaneous IR-PCFR+ guarantees below.
Proofs of the product-of-simplices and simultaneous CFR results established here are collected in Appendices~B and~C.

\subsection{Extensive-form games and CFR}
\label{sec:cfr-setup}
Consider a finite two-player zero-sum extensive-form game with perfect recall. Let $\mathcal I_X$ and $\mathcal I_Y$ denote the information sets of the two players, and let $A(I)$ be the actions available at information set $I$. A behavioral strategy assigns a local distribution to every information set:
\[
x_I\in\Delta_{|A(I)|}\quad(I\in\mathcal I_X),
\qquad
y_J\in\Delta_{|A(J)|}\quad(J\in\mathcal I_Y).
\]
We write $x=(x_I)_{I\in\mathcal I_X}$ and $y=(y_J)_{J\in\mathcal I_Y}$ for the two behavioral strategies.

% We use the standard sequence-form representation~\cite{koller1994,vonstengel1996}. For an information set $I$, let $p_I$ denote its parent sequence, write $Ia$ for the sequence obtained by taking action $a$ at $I$, and let $C_\sigma$
% denote the child information sets immediately following a sequence $\sigma$. If $q^X$ is the realization plan induced by $x$, then
% \[
% q^X[\sigma]
% =
% \prod_{Ia\preceq\sigma}x_I[a],
% \]
% where $Ia\preceq\sigma$ means that the sequence $Ia$ is a prefix of $\sigma$,
% and analogously for $q^Y$. Let $\mathcal Q_X$ and $\mathcal Q_Y$ denote the corresponding sequence-form polytopes.

We use the standard sequence-form representation~\cite{koller1994,vonstengel1996}.
For a player $P\in\{X,Y\}$, sequences record only that player's actions.
For $I\in\mathcal I_P$, let $p_I$ denote its parent sequence, write $Ia$
for the sequence obtained by extending $p_I$ with action $a$ at $I$, and let
$C_\sigma$ denote the information sets of $P$ whose parent sequence is
$\sigma$. If $q^X$ is the realization plan induced by $x$, then for every
$X$-sequence $\sigma$,
\[
q^X[\sigma]
=
\prod_{\substack{Ia\preceq\sigma\\ I\in\mathcal I_X}}x_I[a],
\]
where $Ia\preceq\sigma$ means that the sequence $Ia$ is a prefix of $\sigma$,
and analogously for $q^Y$. Let $\mathcal Q_X$ and $\mathcal Q_Y$ denote the
corresponding sequence-form polytopes.

There is a matrix $A_{\rm seq}$ such that the expected payoff to $X$ is
\[
U(q^X,q^Y)=(q^X)^\top A_{\rm seq}q^Y.
\]
The sequence-form duality gap is
\[
\Gap_{\mathcal Q}(q^X,q^Y)
:=
\max_{\widehat q^X\in\mathcal Q_X}
(\widehat q^X)^\top A_{\rm seq}q^Y
-
\min_{\widehat q^Y\in\mathcal Q_Y}
(q^X)^\top A_{\rm seq}\widehat q^Y.
\]

CFR runs one local learner at each information set~\cite{zinkevich2007}. The utility vector supplied to that learner is its counterfactual utility. In the notation of Zhang et al.~\cite{zhang2026}, these vectors can be defined recursively. For $I\in\mathcal I_X$,
\begin{equation}
F_I^X(x,y)[a]
=
(A_{\rm seq}q^Y)[Ia]
+
\sum_{K\in C_{Ia}}
\ip{F_K^X(x,y)}{x_K},
\label{eq:cfr-counterfactual-X}
\end{equation}
with the $Y$-player counterfactual utilities defined analogously using
$-A_{\rm seq}^{\top}q^X$.
% CFR runs one local regret minimizer at each information set~\cite{zinkevich2007}. The utility vector supplied to a local learner is its counterfactual utility. In the notation of Zhang et al.~\cite{zhang2026}, these vectors can be defined recursively. For $I\in\mathcal I_X$,
% \begin{equation}
% F_I^X(x,y)[a]
% =
% (A_{\rm seq}q^Y)[Ia]
% +
% \sum_{K\in C_{Ia}}
% \ip{F_K^X(x,y)}{x_K},
% \label{eq:cfr-counterfactual-X}
% \end{equation}
% and for $J\in\mathcal I_Y$, writing $Y$ as a utility maximizer,
% \begin{equation}
% F_J^Y(x,y)[a]
% =
% (-A_{\rm seq}^{\top}q^X)[Ja]
% +
% \sum_{K\in C_{Ja}}
% \ip{F_K^Y(x,y)}{y_K}.
% \label{eq:cfr-counterfactual-Y}
% \end{equation}
An empty child sum is zero. Thus each information set induces a standard simplex regret-minimization problem, coupled to the others through the current behavioral profile.

The standard CFR decomposition relates these local counterfactual utilities to the sequence-form game; we recall the underlying identity in Appendix~C.1. A behavioral profile has full reach if each player's realization plan assigns strictly positive probability to the parent sequence of every information set. In particular, Zhang et al.~\cite{zhang2026} show that if the game admits a full-reach Nash equilibrium, then the counterfactual-utility operator satisfies the weighted Minty condition used below, with weights given by the equilibrium parent-sequence realization probabilities. The full-reach condition ensures that all of these weights are strictly positive.

\subsection{Simultaneous IR-PRM+ on a product of simplices}
\label{sec:weighted-minty}
We now establish the product-of-simplices result needed for this reduction under previous-utility prediction. Following the product-space framework of Zhang et al.~\cite{zhang2026}, specialized to the simplex domains on which IR-PRM+ operates, let
\[
\mathcal X=\prod_{i=1}^N\Delta_{k_i},
\qquad
F=(F_1,\ldots,F_N).
\]
In this subsection, $x=(x_1,\ldots,x_N)\in\mathcal X$ denotes the full product
profile. $N$ is the total number of local simplex learners and $k_i$ is the number of actions of learner $i$. In the CFR application, each local learner corresponds to one information set, so
\[
N=|\mathcal I_X|+|\mathcal I_Y|,
\qquad
k_i=|A(I_i)|,
\]
where $I_i$ is the information set represented by learner $i$. Thus each component $x_i\in\Delta_{k_i}$ is one local information-set strategy.

The map $F_i:\mathcal X\to\R^{k_i}$ is bounded and, for some finite $L_i$,
\begin{equation}
\norm{F_i(x)-F_i(x')}_2
\le
L_i\norm{x-x'}_2
\qquad
\text{for all }x,x'\in\mathcal X.
\label{eq:product-lipschitz}
\end{equation}
Each learner $i$ uses the unit-norm initialization
$\tilde r_{i,1}=k_i^{-1/2}\1$, so $\rho_{i,1}=1$. We set $m_{i,1}=0$ and, for $t\ge2$, use
\begin{equation}
u_{i,t}=F_i(x_t),
\qquad
m_{i,t}=F_i(x_{t-1}).
\label{eq:product-prediction}
\end{equation}
Write $\rho_{i,t}$ for its regret-vector norm and $\tilde x_{i,t}$ for the strategy induced before the round-$t$ prediction. Monotonicity gives $\rho_{i,t}\ge1$ for every learner and every round.

% We assume the weighted Minty condition used by Zhang et al.~\cite{zhang2026}: there are $x^\star\in\mathcal X$ and fixed weights $w_i>0$ such that
% \begin{equation}
% \sum_{i=1}^N
% w_i\ip{F_i(x)}{x_i^\star-x_i}
% \ge0
% \qquad
% \text{for every }x\in\mathcal X.
% \label{eq:weighted-minty}
% \end{equation}
The two-player proof above used the fact that one learner's prediction error is controlled by the other learner's consecutive movement. Here the feedback at learner $i$ can depend on every component of $x$, so all local movements can contribute. Lipschitzness first bounds learner $i$'s prediction error by the movement of the full product profile; decomposing that movement across components and applying Corollary~\ref{cor:movement} to each learner gives the following recurrence.

\begin{lemma}[Product-of-simplices norm-growth bound]
\label{lem:product-growth}
For every learner $i$ and every $t\ge2$,
\begin{align}
\rho_{i,t+1}-\rho_{i,t}
\le
\sum_{j=1}^N
&\frac{2(1+\sqrt{k_j})^2(1+\sqrt{k_i})^2L_i^2}
{\rho_{i,t}\rho_{j,t}}
\left[
\frac{\rho_{j,t}}{2k_j}
\norm{x_{j,t}-\tilde x_{j,t}}_2^2
\right.
\notag\\[-0.2em]
&\left.\hspace{8.8em}
+(\rho_{j,t}-\rho_{j,t-1})
\right].
\label{eq:product-growth}
\end{align}
\end{lemma}

The proof extends the argument for Lemma~\ref{lem:norm-growth} to the coupled local learners on the product of simplices; the full proof is given in Appendix~B.1.

To close these coupled recurrences, we use the weighted Minty condition of
Zhang et al.~\cite{zhang2026}: there are $x^\star\in\mathcal X$ and fixed
weights $w_i>0$ such that
\begin{equation}
\sum_{i=1}^N
w_i\ip{F_i(x)}{x_i^\star-x_i}
\ge0
\qquad
\text{for every }x\in\mathcal X.
\label{eq:weighted-minty}
\end{equation}
With $N$ local learners, different regret-vector norms may enter the
large-norm regime at different times. The next theorem shows that this
condition is nevertheless sufficient to bound all of them.

\begin{theorem}[Constant regret under weighted Minty]
\label{thm:weighted-minty}
Suppose each $F_i$ is bounded and Lipschitz and $F$ satisfies the weighted
Minty condition. Then, under simultaneous
IR-PRM+ with previous-utility prediction, every local regret-vector norm
remains bounded over time and every local learner's external regret is
bounded by a finite game-dependent constant independent of $T$.
\end{theorem}

\begin{proof}
Let
\[
w_{\min}:=\min_i w_i,
\qquad
w_{\max}:=\max_i w_i,
\]
and write
\[
c_{ij}:=
2(1+\sqrt{k_j})^2(1+\sqrt{k_i})^2L_i^2.
\]
Since $\rho_{j,t}\ge1$ for every learner $j$, choose a finite threshold
$\Theta_i\ge1$ such that
\[
\frac{c_{ij}}{\Theta_i}
\le
\frac{w_{\min}}{4Nw_{\max}}
\qquad\text{for every }j.
\]
Hence, whenever $\rho_{i,t}\ge\Theta_i$, every coefficient in
\eqref{eq:product-growth} satisfies
\[
\frac{c_{ij}}{\rho_{i,t}\rho_{j,t}}
\le
\frac{w_{\min}}{4Nw_{\max}}.
\]
If learner $i$ never reaches $\Theta_i$, its norm is already bounded;
otherwise monotonicity ensures that it remains above $\Theta_i$ after
its first crossing.

Summing the above-threshold recurrences in
\eqref{eq:product-growth} over the $N$ learners and $T$ rounds removes the factor $1/N$ and telescopes. Thus, after
absorbing the below-threshold  contributions into a finite
constant $C$ independent of $T$,
\[
\sum_i \rho_{i,T+1}
\le
C+
\frac{w_{\min}}{4w_{\max}}
\sum_{t=1}^T\sum_i
\frac{\rho_{i,t}}{2k_i}
\norm{x_{i,t}-\tilde x_{i,t}}_2^2
+
\frac{w_{\min}}{4w_{\max}}
\sum_i\rho_{i,T+1},
\]
where the last term telescopes the norm-growth terms in
\eqref{eq:product-growth}.

It remains to control the squared prediction displacements. Since
$x_i^\star$ is a feasible comparator,
\[
\Reg_i(T)
\ge
\sum_{t=1}^T
\ip{F_i(x_t)}{x_i^\star-x_{i,t}}.
\]
Multiplying by $w_i$, summing over $i$, and using the weighted Minty
condition gives
\[
0\le\sum_i w_i\Reg_i(T).
\]
Fact~\ref{fact:rvu-intermediate} therefore yields
\[
\sum_{t=1}^T\sum_i
w_i\frac{\rho_{i,t}}{2k_i}
\norm{x_{i,t}-\tilde x_{i,t}}_2^2
\le
\sum_i w_i\rho_{i,T+1}.
\]
Using $w_i\ge w_{\min}$ on the left and $w_i\le w_{\max}$ on the right,
\begin{equation}
\sum_{t=1}^T\sum_i
\frac{\rho_{i,t}}{2k_i}
\norm{x_{i,t}-\tilde x_{i,t}}_2^2
\le
\frac{w_{\max}}{w_{\min}}
\sum_i\rho_{i,T+1}.
\label{eq:unweighted-stability}
\end{equation}
Substituting above gives
\[
\sum_i\rho_{i,T+1}
\le
C+\frac14\sum_i\rho_{i,T+1}
+\frac14\sum_i\rho_{i,T+1},
\]
and hence
\[
\sum_i\rho_{i,T+1}\le2C.
\]
Thus every local regret-vector norm remains bounded. Finally,
Fact~\ref{fact:rvu-intermediate} gives
\[
\Reg_i(T)\le\rho_{i,T+1}=\mathcal O_T(1)
\]
for every learner $i$.
\end{proof}

% The resulting bounds also imply square summability of the current local
% residuals on the product of simplices. We state this technical
% consequence in Appendix~B and use it below to control the current CFR profile. 
% The $\mathcal O_T(T^{-1})$ average result, by
% contrast, follows from the bilinear sequence-form zero-sum game.

\subsection{Simultaneous IR-PCFR+}
\label{sec:irpcfr}
We now apply Theorem~\ref{thm:weighted-minty} to the CFR operator from Section~\ref{sec:cfr-setup}, which is bounded and Lipschitz on the
behavioral strategy space~\cite{farina2023instability}. Run one IR-PRM+ learner at every information set, each with the unit-norm initialization \eqref{eq:positive-init}. For $t\ge2$, the predictors at information sets $I\in\mathcal I_X$ and $J\in\mathcal I_Y$ are
\[
m_{I,t}^X=F_I^X(x_{t-1},y_{t-1}),\qquad
m_{J,t}^Y=F_J^Y(x_{t-1},y_{t-1}).
\]
The local strategies at all information sets are formed simultaneously. The counterfactual utilities $F_I^X(x_t,y_t)$ and $F_J^Y(x_t,y_t)$ are then evaluated at that behavioral profile, and all local regret vectors are updated in parallel. Set $\bar q_T^X=T^{-1}\sum_{t=1}^T q_t^X$ and $\bar q_T^Y=T^{-1}\sum_{t=1}^T q_t^Y$.

% The counterfactual-utility operator is bounded and Lipschitz on the
% behavioral strategy space~\cite{farina2023instability}.

% For $I\in\mathcal I_X$, define the local counterfactual regret
% \[
% \Reg_I^X(T)
% :=
% \max_{p\in\Delta_{|A(I)|}}
% \sum_{t=1}^T
% \ip{F_I^X(x_t,y_t)}{p-x_{I,t}},
% \]
% and define $\Reg_J^Y(T)$ analogously. Define the sequence-form external regrets
% \begin{align*}
% \Reg_X^{\mathcal Q}(T)
% &:=
% \max_{\widehat q^X\in\mathcal Q_X}
% \sum_{t=1}^T
% \ip{A_{\rm seq}q_t^Y}{\widehat q^X-q_t^X},\\
% \end{align*}
% and define $\Reg_Y^{\mathcal Q}(T)$ analogously. Finally, 

\begin{theorem}[Constant regret for simultaneous IR-PCFR+]
\label{thm:irpcfr}
In any finite two-player zero-sum perfect-recall extensive-form game admitting a full-reach Nash equilibrium, simultaneous IR-PCFR+ with previous-utility prediction at each information set has both players' sequence-form external regrets bounded by finite game-dependent constants independent of $T$. Consequently,
\[
\Gap_{\mathcal Q}(\bar q_T^X,\bar q_T^Y)=\mathcal O_T(T^{-1}).
\]
\end{theorem}

Theorem~\ref{thm:irpcfr} uses the same local IR-PRM+ analysis as Theorem~\ref{thm:weighted-minty}; the standard CFR decomposition then transfers the local constant-regret guarantees to the global sequence-form game. The full
proof is given in Appendix~C.2.

The product-space movement argument also controls the current local
counterfactual residuals; Appendix~B.2 makes this precise. The
pointwise CFR decomposition then transfers those local residual bounds to the
current sequence-form duality gap; the full proof of the result below is given
in Appendix~C.3.

Let $\mathcal Z_{\mathcal Q}^\star$ denote the sequence-form Nash equilibrium set.
\begin{theorem}[Current sequence-form convergence]
\label{thm:cfr-current}
In the same setting,
\begin{equation}
\sum_{t=1}^{\infty}
\Gap_{\mathcal Q}(q_t^X,q_t^Y)^2<\infty.
\label{eq:cfr-current-square-sum}
\end{equation}
Consequently,
\[
\Gap_{\mathcal Q}(q_t^X,q_t^Y)\to0,
\qquad
\operatorname{dist}\bigl((q_t^X,q_t^Y),\mathcal Z_{\mathcal Q}^\star\bigr)\to0.
\]
Moreover,
\begin{equation}
\min_{1\le t\le T}
\Gap_{\mathcal Q}(q_t^X,q_t^Y)
=
\mathcal O_T(T^{-1/2}).
\label{eq:cfr-best-current}
\end{equation}
\end{theorem}

\section{Alternating Updates}
\label{sec:alternating}

We next consider alternating IR-PRM+. Set $y_0=\tilde y_1$ and form $x_1$ using zero initial prediction. For each $t\ge1$, after $x_t$ has been formed, the column player observes $-A^\top x_t$ against $y_{t-1}$, updates its stored regret vector, and reuses this utility as the prediction producing $y_t$. The row player then observes $Ay_t$ against $x_t$, updates its stored regret vector, and reuses this utility as the prediction producing $x_{t+1}$. Thus each alternating round passes through the intermediate profile $(x_t,y_{t-1})$ before producing the full profile $(x_t,y_t)$.

% For each $t\ge1$, the row player
% updates first, producing $x_t$ while the column strategy remains
% $y_{t-1}$. The column player then updates, producing $y_t$ in response
% to $x_t$. Thus the two updates pass through the intermediate profile
% $(x_t,y_{t-1})$ and produce the full profile $(x_t,y_t)$.

This change in update order introduces two differences from the
simultaneous analysis. First, the two regret-vector growth bounds are
staggered across the row and column updates. Second, the column player's
directly controlled regret is evaluated along the intermediate profiles
$(x_t,y_{t-1})$, whereas its external regret in Theorem~\ref{thm:alt-matrix} is measured along the full
profiles $(x_t,y_t)$. The analysis below handles both effects while
retaining the guarantees of the simultaneous case.

\begin{theorem}[Constant regret with alternation]
\label{thm:alt-matrix}
In any finite two-player zero-sum matrix game, alternating IR-PRM+ with
previous-utility prediction has both players' external regrets bounded by game-dependent
constants independent of $T$. Consequently,
\[
\Gap(\bar x_T,\bar y_T)=\mathcal O_T(T^{-1}).
\]
\end{theorem}

\paragraph{Proof idea.}
The alternating updates lead to a staggered pair of norm-growth
recurrences. The column recurrence contains the reciprocal factor
$1/(\rho_t^X\rho_t^Y)$, while the row recurrence contains
$1/(\rho_t^X\rho_{t+1}^Y)$. Since $\rho_t^Y$ is nondecreasing, the row coefficient is no larger than the column coefficient. The same threshold used in the simultaneous analysis therefore controls both, and summing
them in alternating order gives bounded regret-vector norms.

For the column player, this argument initially gives bounded regret on
the intermediate profiles $(x_t,y_{t-1})$. Replacing $y_{t-1}$ by $y_t$
changes the cumulative regret by an amount controlled by
$\rho_{t+1}^Y-\rho_t^Y$ at each index $t$. These increments telescope, so
boundedness of $\rho_t^Y$ yields bounded full-profile regret as well. The
complete argument is given in Appendix~D.1.

The same bounds imply square summability of the consecutive strategy movements needed
to compare consecutive predictions, which gives convergence of the
current profiles. The full argument is given in Appendix~D.2.

\begin{theorem}[Current-profile convergence]
\label{thm:alt-matrix-current}
In the same setting,
\begin{equation}
\sum_{t=1}^{\infty}\Gap(x_t,y_t)^2<\infty.
\label{eq:alt-matrix-gap-square-sum}
\end{equation}
Consequently,
\[
\Gap(x_t,y_t)\to0,
\qquad
\operatorname{dist}\bigl((x_t,y_t),\mathcal Z^\star\bigr)\to0.
\]
Moreover,
\begin{equation}
\min_{1\le t\le T}\Gap(x_t,y_t)=\mathcal O_T(T^{-1/2}).
\label{eq:alt-matrix-best-current}
\end{equation}
\end{theorem}

\subsection{Alternating IR-PCFR+}
The CFR decomposition itself is unchanged, but alternation staggers the local counterfactual utilities. Define the intermediate and full behavioral profiles by
\[
h_t=(x_t,y_{t-1}),
\qquad
z_t=(x_t,y_t).
\]
At each column-player information set $J$, the learner observes $F_J^Y(h_t)$ against $y_{J,t-1}$, updates its stored regret vector, and reuses this utility as the prediction producing $y_{J,t}$. This forms the full profile $z_t$. Each row-player learner at information set $I$ then observes $F_I^X(z_t)$ against $x_{I,t}$, updates its stored regret vector, and reuses this utility as the prediction producing $x_{I,t+1}$.

Thus the row-player local regret is controlled along the full profiles $z_t$, whereas the column-player local regret is first controlled along the intermediate profiles $h_t$. The remaining step is to handle this staggered weighted comparison and transfer the column half-step bounds to sequence-form regret along the full profiles.

\begin{theorem}[Constant regret for alternating IR-PCFR+]
\label{thm:alt-cfr}
In any finite two-player zero-sum perfect-recall extensive-form game admitting a full-reach Nash equilibrium, alternating IR-PCFR+ with previous-utility prediction at each information set has both players' sequence-form external regrets bounded by finite game-dependent constants independent of $T$. Consequently,
\[
\Gap_{\mathcal Q}(\bar q_T^X,\bar q_T^Y)=\mathcal O_T(T^{-1}).
\]
\end{theorem}

\paragraph{Proof idea.}
The staggered matrix-game argument is applied to the local counterfactual-utility learners. The full-reach equilibrium weights compare row-player deviations at $z_t$ with column-player deviations at $h_t$, so an additional term appears from changing the column strategy from $y_{t-1}$ to $y_t$. For each column learner, this payoff change is bounded below by the negative increase in its regret-vector norm. Summing over time therefore telescopes, allowing the staggered weighted comparison to yield bounded local regret. The standard CFR decomposition transfers the row bounds directly to full-profile sequence-form regret, while the same payoff-shift bound converts the column half-step regret to full-profile regret. Details are in Appendix~E.1.

The same local bounds also yield control of the current full profiles; the
proof is given in Appendix~E.2.

\begin{theorem}[Current sequence-form convergence]
\label{thm:alt-cfr-current}
In the same setting,
\begin{equation}
\sum_{t=1}^{\infty}\Gap_{\mathcal Q}(q_t^X,q_t^Y)^2<\infty.
\label{eq:alt-cfr-current-square-sum}
\end{equation}
Consequently,
\[
\Gap_{\mathcal Q}(q_t^X,q_t^Y)\to0,
\qquad
\operatorname{dist}\bigl((q_t^X,q_t^Y),\mathcal Z_{\mathcal Q}^\star\bigr)\to0.
\]
Moreover,
\begin{equation}
\min_{1\le t\le T}
\Gap_{\mathcal Q}(q_t^X,q_t^Y)
=
\mathcal O_T(T^{-1/2}).
\label{eq:alt-cfr-best-current}
\end{equation}
\end{theorem}

\section{Conclusion and Future Work}
We showed that previous-utility prediction lets IR-PRM+ recover the constant-regret and convergence guarantees of extra-gradient IREG-PRM+ under both simultaneous and alternating updates, with corresponding guarantees for IR-PCFR+. The key step is a post-update displacement bound linking strategy movement to regret-vector growth. We do not repeat an empirical evaluation, since the practical performance of these IR-PRM+ variants has already been evaluated by Zhang et al.~\cite{zhang2026}; our contribution is to provide the missing theoretical guarantees. Future work is to understand how far these guarantees extend beyond the zero-sum and weighted-Minty settings considered here, and whether the full-reach assumption in the CFR analysis can be weakened.

\appendix
\numberwithin{equation}{section}

\section{Proofs for the Simultaneous Matrix-Game Analysis}
\label{app:sim-matrix-proofs}

\subsection{Proof of Fact~\ref{fact:rvu-intermediate}}
\begin{proof}
The IR-PRM+ RVU analysis yields the following bound before the negative
prediction displacement term is discarded. Let $R_T$ denote the cumulative regret vector. The proof of Zhang et
al.~\cite{zhang2026}, before eliminating the terminal regret-vector norm,
gives the coordinatewise inequality
\[
R_T\le \tilde r_{T+1}-
\left(\sum_{t=1}^T\frac{\rho_t}{2k}\norm{x_t-\tilde x_t}_2^2\right)\1.
\]
External regret is the largest coordinate of $R_T$. Since the quantity in
parentheses is the same in every coordinate,
\begin{align*}
\Reg(T)
&=\max_a R_{T,a}\\
&\le \norm{\tilde r_{T+1}}_\infty
-
\sum_{t=1}^T\frac{\rho_t}{2k}\norm{x_t-\tilde x_t}_2^2\\
&\le \rho_{T+1}
-
\sum_{t=1}^T\frac{\rho_t}{2k}\norm{x_t-\tilde x_t}_2^2,
\end{align*}
where the last step uses
$\norm{\tilde r_{T+1}}_\infty\le\norm{\tilde r_{T+1}}_2=\rho_{T+1}$.
This is \eqref{eq:rvu-intermediate}.
\end{proof}

\subsection{Predicted-utility residual bound}
\label{app:prediction-residual}
The following auxiliary lemma converts a small prediction displacement into approximate optimality for the predicted utility. It is used only for the current-profile results, not for the constant-regret proof.

\begin{lemma}[Predicted-utility residual bound]
\label{lem:prediction-residual}
For a $k$-action IR-PRM+ learner,
\begin{equation}
\max_{p\in\Delta_k}\ip{m_t}{p-x_t}
\le
2\sqrt{2k}\,\rho_t\norm{x_t-\tilde x_t}_2.
\label{eq:prediction-residual}
\end{equation}
\end{lemma}

\begin{proof}
By \eqref{eq:rho-lower-one}, $\rho_t>0$. Fix an arbitrary comparator
$p\in\Delta_k$. We will first bound the residual for this fixed $p$ and
then maximize over $p$.

From
$r_t=\tilde r_t+m_t-\gamma_t\1$, define
\[
q_t:=\pos{r_t}-r_t\ge0.
\]
Rearranging gives
\[
m_t=\gamma_t\1+(\pos{r_t}-\tilde r_t)-q_t.
\]
The vector $q_t$ is supported only on coordinates where $r_t<0$, while
$x_t$ is supported on coordinates where $\pos{r_t}>0$. Hence their
supports are disjoint and $\ip{q_t}{x_t}=0$. Also
$\ip{\1}{p-x_t}=0$ because both $p$ and $x_t$ lie in the simplex. Thus
\begin{align*}
\ip{m_t}{p-x_t}
&=
\ip{\pos{r_t}-\tilde r_t}{p-x_t}
-
\ip{q_t}{p-x_t}\\
&=
\ip{\pos{r_t}-\tilde r_t}{p-x_t}
-
\ip{q_t}{p}\\
&\le
\ip{\pos{r_t}-\tilde r_t}{p-x_t}.
\end{align*}
The last term is controlled by Cauchy--Schwarz. Since any two simplex
points are at Euclidean distance at most $\sqrt2$,
\begin{equation}
\ip{m_t}{p-x_t}
\le
\sqrt2\,\norm{\pos{r_t}-\tilde r_t}_2.
\label{eq:prediction-residual-first-step}
\end{equation}

It remains to relate the difference of the two regret vectors to the
prediction displacement $x_t-\tilde x_t$. The prediction step preserves
Euclidean norm, so
$\norm{\pos{r_t}}_2=\norm{\tilde r_t}_2=\rho_t$. Because $x_t$ and
$\tilde x_t$ are their $\ell_1$ normalizations, equivalently
\[
\pos{r_t}=\rho_t\frac{x_t}{\norm{x_t}_2},
\qquad
\tilde r_t=\rho_t\frac{\tilde x_t}{\norm{\tilde x_t}_2}.
\]
Let $a=\norm{x_t}_2$ and $b=\norm{\tilde x_t}_2$. Adding and subtracting
$\rho_t\tilde x_t/a$ gives
\begin{align*}
\norm{\pos{r_t}-\tilde r_t}_2
&=
\rho_t\norm{\frac{x_t}{a}-\frac{\tilde x_t}{b}}_2\\
&\le
\frac{\rho_t}{a}\norm{x_t-\tilde x_t}_2
+
\rho_t\left|\frac1a-\frac1b\right|\norm{\tilde x_t}_2\\
&=
\frac{\rho_t}{a}\norm{x_t-\tilde x_t}_2
+
\frac{\rho_t|b-a|}{a}\\
&\le
\frac{2\rho_t}{a}\norm{x_t-\tilde x_t}_2.
\end{align*}
The last step uses the reverse triangle inequality
$|b-a|\le\norm{x_t-\tilde x_t}_2$. Finally,
$x_t\in\Delta_k$ implies $a=\norm{x_t}_2\ge1/\sqrt{k}$, so
\begin{equation}
\norm{\pos{r_t}-\tilde r_t}_2
\le
2\sqrt{k}\,\rho_t\norm{x_t-\tilde x_t}_2.
\label{eq:prediction-residual-vector-step}
\end{equation}
Combining \eqref{eq:prediction-residual-first-step} and
\eqref{eq:prediction-residual-vector-step} gives
\[
\ip{m_t}{p-x_t}
\le
2\sqrt{2k}\,\rho_t\norm{x_t-\tilde x_t}_2.
\]
The bound is independent of $p$, so maximizing over $p\in\Delta_k$
proves \eqref{eq:prediction-residual}.
\end{proof}

\subsection{Proof of Theorem~\ref{thm:current-matrix}}
\begin{proof}
We first show that the prediction and post-update displacements, and hence the
consecutive strategy movements, are square summable. We then use the residual lemma
for the predicted utility to convert these bounds into square summability of the current best-response residuals and,
hence, of the current duality gap.

The proof of Theorem~\ref{thm:main} gives two facts that we will repeatedly
use: the regret-vector norms remain bounded over time, and
$\sum_t H_t<\infty$. Expanding the definition of $H_t$,
\[
\sum_{t=1}^{\infty}
\frac{\rho_t^X}{2k^X}\norm{x_t-\tilde x_t}_2^2<\infty,
\qquad
\sum_{t=1}^{\infty}
\frac{\rho_t^Y}{2k^Y}\norm{y_t-\tilde y_t}_2^2<\infty.
\]
Since $\rho_t^X,\rho_t^Y\ge1$, the factors $\rho_t^X$ and $\rho_t^Y$ can be removed directly, giving
\begin{equation}
\sum_{t=1}^\infty\norm{x_t-\tilde x_t}_2^2<\infty,
\qquad
\sum_{t=1}^\infty\norm{y_t-\tilde y_t}_2^2<\infty.
\label{eq:matrix-prediction-movement-sum}
\end{equation}

The post-update displacements are square summable for the same reason. For example, Lemma~\ref{lem:update-change} and $\rho_{t+1}^X\ge1$ give
\[
\norm{\tilde x_{t+1}-x_t}_2^2
\le
2(1+\sqrt{k^X})^2
(\rho_{t+1}^X-\rho_t^X).
\]
The norm increments telescope, and the norm sequence is bounded, so summing
the right-hand side over time is finite. The same argument applies to the column player;
thus
\[
\sum_{t=1}^\infty\norm{\tilde x_{t+1}-x_t}_2^2<\infty,
\qquad
\sum_{t=1}^\infty\norm{\tilde y_{t+1}-y_t}_2^2<\infty.
\]
To recover the movement between consecutive played strategies, insert the
intermediate strategies and use
$\norm{a+b}_2^2\le2\norm a_2^2+2\norm b_2^2$. Together with the two
square-summability statements above, this gives
\begin{equation}
\sum_{t=2}^\infty\norm{x_t-x_{t-1}}_2^2<\infty,
\qquad
\sum_{t=2}^\infty\norm{y_t-y_{t-1}}_2^2<\infty.
\label{eq:matrix-full-movement-sum}
\end{equation}

We now translate these square-summability bounds into current best-response residuals. Define the row and column
best-response residuals
\[
R_t^X:=\max_{x'\in\Delta_{k^X}}\ip{Ay_t}{x'-x_t},
\qquad
R_t^Y:=\max_{y'\in\Delta_{k^Y}}\ip{-A^\top x_t}{y'-y_t}.
\]
For $t\ge2$, consider first the row player, for which the previous-utility predictor is
$m_t^X=Ay_{t-1}$. Add and subtract this prediction inside the inner product.
For any $x'\in\Delta_{k^X}$,
\begin{align*}
\ip{Ay_t}{x'-x_t}
&=
\ip{m_t^X}{x'-x_t}
+
\ip{A(y_t-y_{t-1})}{x'-x_t}\\
&\le
2\sqrt{2k^X}\,\rho_t^X\norm{x_t-\tilde x_t}_2
+
\sqrt2\norm A_2\norm{y_t-y_{t-1}}_2.
\end{align*}
The first term is controlled by Lemma~\ref{lem:prediction-residual}. The second uses
Cauchy--Schwarz, the operator-norm bound for $A$, and the simplex diameter
$\sqrt2$. Now define
\[
C_X:=2\sqrt{2k^X}\sup_t\rho_t^X.
\]
Maximizing over $x'$ gives
\begin{align}
R_t^X
&\le C_X\norm{x_t-\tilde x_t}_2
+\sqrt2\norm A_2\norm{y_t-y_{t-1}}_2.
\label{eq:matrix-current-X}
\end{align}
For the column player, the identical calculation with
$m_t^Y=-A^\top x_{t-1}$ gives
\begin{align}
R_t^Y
&\le C_Y\norm{y_t-\tilde y_t}_2
+\sqrt2\norm A_2\norm{x_t-x_{t-1}}_2,
\label{eq:matrix-current-Y}
\end{align}
where $C_Y:=2\sqrt{2k^Y}\sup_t\rho_t^Y$.

Each term on the right of \eqref{eq:matrix-current-X}--
\eqref{eq:matrix-current-Y} is square summable by
\eqref{eq:matrix-prediction-movement-sum} and
\eqref{eq:matrix-full-movement-sum}. The initial residual terms are finite and contribute only a bounded constant. Using
$(a+b)^2\le2a^2+2b^2$ therefore gives
\[
\sum_t(R_t^X)^2<\infty,
\qquad
\sum_t(R_t^Y)^2<\infty.
\]
The matrix-game duality gap is exactly the sum of these two residuals,
$\Gap(x_t,y_t)=R_t^X+R_t^Y$. A final application of
$(a+b)^2\le2a^2+2b^2$ yields \eqref{eq:matrix-gap-square-sum}.

Square summability implies $\Gap(x_t,y_t)\to0$. To prove that the distance to the Nash equilibrium set converges to zero, suppose instead that the distance does not tend
to zero. Then there are $\varepsilon>0$ and a subsequence
$(x_{t_j},y_{t_j})$ satisfying
$\operatorname{dist}((x_{t_j},y_{t_j}),\mathcal Z^\star)\ge\varepsilon$.
Compactness of $\Delta_{k^X}\times\Delta_{k^Y}$ gives a convergent subsubsequence.
Its limit has zero gap by continuity and hence lies in $\mathcal Z^\star$,
contradicting the distance lower bound. Therefore
$\operatorname{dist}((x_t,y_t),\mathcal Z^\star)\to0$.

Finally, let
$K:=\sum_{t=1}^\infty\Gap(x_t,y_t)^2<\infty$. Since every term is
nonnegative,
\[
T\min_{1\le t\le T}\Gap(x_t,y_t)^2
\le
\sum_{t=1}^T\Gap(x_t,y_t)^2
\le K.
\]
Taking square roots gives \eqref{eq:matrix-best-current}.
\end{proof}

\section{Proofs for the Product-of-Simplices Analysis}
\label{app:product-proofs}

\subsection{Proof of Lemma~\ref{lem:product-growth}}
\begin{proof}
Fix a learner $i$ and a round $t\ge2$. The purpose of the calculation is to
express learner $i$'s norm growth in terms of nonnegative quantities $B_{j,t}$, one for each learner $j$. Previous-utility prediction gives
\[
u_{i,t}-m_{i,t}=F_i(x_t)-F_i(x_{t-1}).
\]
Applying the bound \eqref{eq:center-bound} and then Lipschitzness,
\begin{align}
\norm{g_{i,t}}_2^2
&\le
(1+\sqrt{k_i})^2
\norm{F_i(x_t)-F_i(x_{t-1})}_2^2
\notag\\
&\le
(1+\sqrt{k_i})^2L_i^2
\norm{x_t-x_{t-1}}_2^2.
\label{eq:product-g-square}
\end{align}
Since we use the product Euclidean norm,
\[
\norm{x_t-x_{t-1}}_2^2
=
\sum_{j=1}^N\norm{x_{j,t}-x_{j,t-1}}_2^2.
\]
Thus it remains to control each component's movement separately.

Define
\[
B_{j,t}:=
\frac{\rho_{j,t}}{2k_j}\norm{x_{j,t}-\tilde x_{j,t}}_2^2
+(\rho_{j,t}-\rho_{j,t-1}).
\]
Corollary~\ref{cor:movement} gives
\begin{align*}
\norm{x_{j,t}-x_{j,t-1}}_2^2
&\le
2\norm{x_{j,t}-\tilde x_{j,t}}_2^2
+
\frac{4(1+\sqrt{k_j})^2}{\rho_{j,t}}
(\rho_{j,t}-\rho_{j,t-1}).
\end{align*}
We bound both terms using the same quantity $B_{j,t}$. Since
$(1+\sqrt{k_j})^2\ge k_j$,
\[
2\norm{x_{j,t}-\tilde x_{j,t}}_2^2
\le
\frac{4(1+\sqrt{k_j})^2}{\rho_{j,t}}
\frac{\rho_{j,t}}{2k_j}
\norm{x_{j,t}-\tilde x_{j,t}}_2^2.
\]
Therefore
\begin{equation}
\norm{x_{j,t}-x_{j,t-1}}_2^2
\le
\frac{4(1+\sqrt{k_j})^2}{\rho_{j,t}}B_{j,t}.
\label{eq:product-block-movement}
\end{equation}

Substituting \eqref{eq:product-block-movement} into the component decomposition
of \eqref{eq:product-g-square} gives
\[
\norm{g_{i,t}}_2^2
\le
\sum_{j=1}^N
\frac{4(1+\sqrt{k_j})^2(1+\sqrt{k_i})^2L_i^2}{\rho_{j,t}}
B_{j,t}.
\]
Finally, the one-step norm-growth inequality \eqref{eq:growth} for learner
$i$ divides this expression by $2\rho_{i,t}$:
\[
\rho_{i,t+1}-\rho_{i,t}
\le
\sum_{j=1}^N
\frac{2(1+\sqrt{k_j})^2(1+\sqrt{k_i})^2L_i^2}
{\rho_{i,t}\rho_{j,t}}B_{j,t}.
\]
Expanding the definition of $B_{j,t}$ is exactly
\eqref{eq:product-growth}.
\end{proof}

\subsection{Current residual under weighted Minty}
\label{app:weighted-current}
\begin{corollary}[Current residual under weighted Minty]
\label{cor:weighted-current}
Under the assumptions of Theorem~\ref{thm:weighted-minty},
\begin{equation}
\sum_{t=1}^{\infty}
\left(
\sum_{i=1}^{N}
\max_{p_i\in\Delta_{k_i}}
\ip{F_i(x_t)}{p_i-x_{i,t}}
\right)^2
<\infty.
\label{eq:weighted-current-square-sum}
\end{equation}
Consequently,
\[
\sum_{i=1}^{N}
\max_{p_i\in\Delta_{k_i}}
\ip{F_i(x_t)}{p_i-x_{i,t}}
\to0,
\]
and
\begin{equation}
\min_{1\le t\le T}
\sum_{i=1}^{N}
\max_{p_i\in\Delta_{k_i}}
\ip{F_i(x_t)}{p_i-x_{i,t}}
=
\mathcal O_T(T^{-1/2}).
\label{eq:weighted-current-best}
\end{equation}
\end{corollary}

\begin{proof}
Theorem~\ref{thm:weighted-minty} gives local norms that remain bounded over time, and
\eqref{eq:unweighted-stability} gives a finite total stability sum:
\begin{equation}
\sum_{t=1}^{\infty}\sum_{i=1}^{N}
\frac{\rho_{i,t}}{2k_i}\norm{x_{i,t}-\tilde x_{i,t}}_2^2<\infty.
\label{eq:weighted-pred-movement}
\end{equation}
Since $\rho_{i,t}\ge1$ for every $i$, each prediction displacement is square summable. For its post-update displacement,
Lemma~\ref{lem:update-change} gives
\[
\norm{\tilde x_{i,t+1}-x_{i,t}}_2^2
\le
2(1+\sqrt{k_i})^2
(\rho_{i,t+1}-\rho_{i,t}),
\]
and the right-hand side is summable because the bounded monotone norm
increments telescope. For each $i$, inserting $\tilde x_{i,t}$
between $x_{i,t-1}$ and $x_{i,t}$ and using
$\|a+b\|_2^2\le 2\|a\|_2^2+2\|b\|_2^2$, then summing over the $N$
components, yields
\begin{equation}
\sum_{t=2}^{\infty}\norm{x_t-x_{t-1}}_2^2<\infty.
\label{eq:product-full-movement-sum}
\end{equation}

For $t\ge2$, we next bound the current residual for learner $i$. Add and
subtract the previous-utility prediction $F_i(x_{t-1})$:
\begin{align*}
&\max_{p_i\in\Delta_{k_i}}
\ip{F_i(x_t)}{p_i-x_{i,t}}\\
&\quad\le
\max_{p_i\in\Delta_{k_i}}
\ip{F_i(x_{t-1})}{p_i-x_{i,t}}\\
&\qquad+
\max_{p_i\in\Delta_{k_i}}
\ip{F_i(x_t)-F_i(x_{t-1})}{p_i-x_{i,t}}.
\end{align*}
Lemma~\ref{lem:prediction-residual} bounds the first term. For the second,
Cauchy--Schwarz, Lipschitzness, and
$\norm{p_i-x_{i,t}}_2\le\sqrt2$ give
\[
\max_{p_i}\ip{F_i(x_t)-F_i(x_{t-1})}{p_i-x_{i,t}}
\le
\sqrt2L_i\norm{x_t-x_{t-1}}_2.
\]
Therefore
\begin{align}
\max_{p_i\in\Delta_{k_i}}\ip{F_i(x_t)}{p_i-x_{i,t}}
&\le 2\sqrt{2k_i}\,\rho_{i,t}
\norm{x_{i,t}-\tilde x_{i,t}}_2
\notag\\
&\quad+\sqrt2L_i\norm{x_t-x_{t-1}}_2.
\label{eq:product-current-local}
\end{align}
Both sequences on the right are square summable: the first by
\eqref{eq:weighted-pred-movement} and the time-independent norm bound, and the second
by \eqref{eq:product-full-movement-sum}. Hence the local residual is square
summable for $t\ge2$ and the single, first iteration residual is finite.

Let $a_{i,t}$ denote the nonnegative local residual in
\eqref{eq:product-current-local}. Since there are $N$ local residuals,
Cauchy--Schwarz in $\mathbb R^N$ gives
\[
\left(\sum_{i=1}^Na_{i,t}\right)^2
\le
N\sum_{i=1}^Na_{i,t}^2.
\]
Summing over $t$ and using the square summability of every $a_{i,t}$ proves
\eqref{eq:weighted-current-square-sum}. In particular the summed residual
converges to zero.

Finally, if
$K:=\sum_t(\sum_i a_{i,t})^2<\infty$, then
\[
T\min_{1\le t\le T}\left(\sum_i a_{i,t}\right)^2
\le
\sum_{t=1}^T\left(\sum_i a_{i,t}\right)^2
\le K.
\]
Taking square roots yields \eqref{eq:weighted-current-best}.
\end{proof}

\section{Proofs for Simultaneous IR-PCFR+}
\label{app:sim-cfr-proofs}

For $I\in\mathcal I_X$, define the local counterfactual regret
\[
\Reg_I^X(T)
:=
\max_{p\in\Delta_{|A(I)|}}
\sum_{t=1}^T
\ip{F_I^X(x_t,y_t)}{p-x_{I,t}},
\]
and define $\Reg_J^Y(T)$ analogously. The sequence-form external regrets are
\[
\Reg_X^{\mathcal Q}(T)
:=
\max_{\widehat q^X\in\mathcal Q_X}
\sum_{t=1}^T
\ip{A_{\rm seq}q_t^Y}{\widehat q^X-q_t^X},
\]
with $\Reg_Y^{\mathcal Q}(T)$ defined analogously.

\subsection{Standard CFR decomposition}
\label{app:cfr-decomposition}
The counterfactual utility vectors are defined so that a global sequence-form deviation decomposes into weighted local behavioral deviations. The following pointwise identity is the standard local-to-global relation underlying the CFR regret decomposition~\cite{zinkevich2007,farina2019,zhang2026}.

\begin{fact}[Standard pointwise CFR decomposition]
\label{fact:cfr-decomposition}
\hfill\break
Let $(x,y)$ induce $(q^X,q^Y)$. Let $\widehat x$ be any comparator behavioral strategy inducing $\widehat q^X$. Then
\begin{equation}
\ip{A_{\rm seq}q^Y}{\widehat q^X-q^X}
=
\sum_{I\in\mathcal I_X}
\widehat q^X[p_I]
\ip{F_I^X(x,y)}{\widehat x_I-x_I}.
\label{eq:cfr-decomposition-X}
\end{equation}
Similarly, for any comparator $\widehat y$ inducing $\widehat q^Y$,
\begin{equation}
\ip{-A_{\rm seq}^{\top}q^X}{\widehat q^Y-q^Y}
=
\sum_{J\in\mathcal I_Y}
\widehat q^Y[p_J]
\ip{F_J^Y(x,y)}{\widehat y_J-y_J}.
\label{eq:cfr-decomposition-Y}
\end{equation}
\end{fact}
A behavioral profile has \emph{full reach} if each player's realization plan assigns strictly positive probability to the parent sequence of every information set of that player. Suppose $(q^{X,\star},q^{Y,\star})$ is the realization-plan pair induced by a full-reach Nash equilibrium, with corresponding behavioral strategies $(x^\star,y^\star)$. Applying the pointwise decomposition with these equilibrium comparators gives

\begin{align}
&\sum_{I\in\mathcal I_X}
q^{X,\star}[p_I]
\ip{F_I^X(x,y)}{x_I^\star-x_I}
\notag\\[-0.2em]
&\quad+
\sum_{J\in\mathcal I_Y}
q^{Y,\star}[p_J]
\ip{F_J^Y(x,y)}{y_J^\star-y_J}
\notag\\
&\qquad=
U(q^{X,\star},q^Y)-U(q^X,q^{Y,\star})
\ge0.
\label{eq:cfr-weighted-minty}
\end{align}
The inequality is saddle-point optimality. Under the full-reach condition, every parent-sequence weight $q^{X,\star}[p_I]$ and $q^{Y,\star}[p_J]$ is strictly positive. Hence the CFR counterfactual-utility operator satisfies the weighted Minty condition used in Section~\ref{sec:weighted-minty}; this is the full-reach CFR reduction of Zhang et al.~\cite{zhang2026}.

Summing \eqref{eq:cfr-decomposition-X}--\eqref{eq:cfr-decomposition-Y} over time and upper-bounding each local cumulative deviation by the corresponding local external regret gives the standard cumulative CFR regret decomposition; cf. Proposition~6.1 of Zhang et al.~\cite{zhang2026}. At a single iteration, the same pointwise identities transfer bounds on local residuals to the sequence-form residuals.

\subsection{Proof of Theorem~\ref{thm:irpcfr}}
\begin{proof}
Equation~\eqref{eq:cfr-weighted-minty} supplies the weighted Minty condition, while the boundedness and Lipschitzness of the counterfactual-utility operator stated in Section~\ref{sec:irpcfr} supply the remaining assumptions of Theorem~\ref{thm:weighted-minty}. Hence every local counterfactual regret is $\mathcal O_T(1)$.We now transfer those local bounds to the sequence-form external regrets.

Fix a sequence-form comparator $\widehat q^X$ induced by a behavioral
strategy $\widehat x$. Summing the pointwise CFR decomposition
\eqref{eq:cfr-decomposition-X} over $t$ gives
\[
\sum_{t=1}^T\ip{A_{\rm seq}q_t^Y}{\widehat q^X-q_t^X}
=
\sum_{I\in\mathcal I_X}\widehat q^X[p_I]
\sum_{t=1}^T
\ip{F_I^X(x_t,y_t)}{\widehat x_I-x_{I,t}}.
\]
For each information set $I$, the inner sum is one particular fixed-action
comparison and is therefore at most the local external regret
$\Reg_I^X(T)$. Since $0\le\widehat q^X[p_I]\le1$,
\begin{align*}
\sum_{t=1}^T\ip{A_{\rm seq}q_t^Y}{\widehat q^X-q_t^X}
&\le
\sum_I\widehat q^X[p_I]\Reg_I^X(T)\\
&\le
\sum_I\pos{\Reg_I^X(T)}
=\mathcal O_T(1).
\end{align*}
The bound is uniform in the comparator $\widehat q^X$, so maximizing over
all sequence-form comparators gives
$\Reg_X^{\mathcal Q}(T)=\mathcal O_T(1)$. The same argument for the column player gives
$\Reg_Y^{\mathcal Q}(T)=\mathcal O_T(1)$.

For completeness, we now make the regret-to-gap identity explicit. By
linearity in the comparator,
\begin{align*}
\Reg_X^{\mathcal Q}(T)
&=
T\max_{\widehat q^X}
\ip{A_{\rm seq}\bar q_T^Y}{\widehat q^X}
-
\sum_{t=1}^T U(q_t^X,q_t^Y),\\
\Reg_Y^{\mathcal Q}(T)
&=
\sum_{t=1}^T U(q_t^X,q_t^Y)
-
T\min_{\widehat q^Y}
\ip{A_{\rm seq}\widehat q^Y}{\bar q_T^X}.
\end{align*}
The realized-payoff sums cancel when the two equations are added. The
remaining two optimization terms are exactly the sequence-form duality gap,
so
\begin{equation}
\Reg_X^{\mathcal Q}(T)+\Reg_Y^{\mathcal Q}(T)
=T\Gap_{\mathcal Q}(\bar q_T^X,\bar q_T^Y).
\label{eq:cfr-gap-identity}
\end{equation}
Since the left side is bounded independently of $T$, division by $T$ gives
$\Gap_{\mathcal Q}(\bar q_T^X,\bar q_T^Y)=\mathcal O_T(T^{-1})$.
\end{proof}

\subsection{Proof of Theorem~\ref{thm:cfr-current}}
\begin{proof}
The product-of-simplices corollary already gives square summability of the
sum of the local current residuals. The only remaining step is to show
that the sequence-form duality gap is no larger than that local
sum.

Write
\begin{align*}
R_{I,t}^X&:=\max_{p\in\Delta_{|A(I)|}}
\ip{F_I^X(x_t,y_t)}{p-x_{I,t}},\\
R_{J,t}^Y&:=\max_{p\in\Delta_{|A(J)|}}
\ip{F_J^Y(x_t,y_t)}{p-y_{J,t}},
\end{align*}
and set
\[
R_t:=\sum_{I\in\mathcal I_X}R_{I,t}^X+
\sum_{J\in\mathcal I_Y}R_{J,t}^Y.
\]
Corollary~\ref{cor:weighted-current} gives
\begin{equation}
\sum_{t=1}^{\infty}R_t^2<\infty.
\label{eq:cfr-local-current-square-sum}
\end{equation}

Fix a round $t$ and a sequence-form row comparator $\widehat q^X$ induced
by $\widehat x$. The pointwise decomposition gives
\[
\ip{A_{\rm seq}q_t^Y}{\widehat q^X-q_t^X}
=
\sum_{I\in\mathcal I_X}
\widehat q^X[p_I]
\ip{F_I^X(x_t,y_t)}{\widehat x_I-x_{I,t}}.
\]
For every $I$, the local inner product is at most $R_{I,t}^X$, and the
realization weight satisfies $0\le\widehat q^X[p_I]\le1$. Therefore
\[
\ip{A_{\rm seq}q_t^Y}{\widehat q^X-q_t^X}
\le
\sum_{I\in\mathcal I_X}R_{I,t}^X.
\]
The right side does not depend on the comparator, so maximizing over
$\widehat q^X$ yields
\[
\max_{\widehat q^X}
\ip{A_{\rm seq}q_t^Y}{\widehat q^X-q_t^X}
\le
\sum_{I\in\mathcal I_X}R_{I,t}^X.
\]
The same argument for the column player gives
\[
\max_{\widehat q^Y}
\ip{-A_{\rm seq}^{\top}q_t^X}{\widehat q^Y-q_t^Y}
\le
\sum_{J\in\mathcal I_Y}R_{J,t}^Y.
\]
The sequence-form duality gap is the sum of these two sequence-form residuals.
Hence
\[
0\le\Gap_{\mathcal Q}(q_t^X,q_t^Y)\le R_t.
\]
Squaring and summing, \eqref{eq:cfr-local-current-square-sum}
gives \eqref{eq:cfr-current-square-sum}.

It follows that the current sequence-form gap converges to zero. For the
distance conclusion, suppose instead that
\[
\operatorname{dist}\bigl((q_t^X,q_t^Y),\mathcal Z_{\mathcal Q}^\star\bigr)
\]
did not tend to zero. Then some $\varepsilon>0$ and subsequence would remain at
least $\varepsilon$ away from the sequence-form Nash equilibrium set. The realization-
plan polytopes are compact, so this subsequence has a convergent
subsubsequence. Continuity of $\Gap_{\mathcal Q}$ makes the limit have zero
gap, hence the limit belongs to $\mathcal Z_{\mathcal Q}^\star$, a
contradiction.

Finally, if
$K:=\sum_t\Gap_{\mathcal Q}(q_t^X,q_t^Y)^2<\infty$, then
\[
T\min_{1\le t\le T}\Gap_{\mathcal Q}(q_t^X,q_t^Y)^2
\le K.
\]
Taking square roots gives \eqref{eq:cfr-best-current}.
\end{proof}

\section{Proofs for Alternating IR-PRM+}
\label{app:alt-matrix}
We use the alternating timing from Section~\ref{sec:alternating}, including $y_0=\tilde y_1$ and the first column update at $(x_1,y_0)$. For the row player, $\tilde x_t$ is the strategy induced by its regret vector before the prediction producing $x_t$. For the column player, the indexing is shifted by one half-step: $y_{t-1}$ is the strategy present when $-A^\top x_t$ is observed, the update produces the next stored strategy $\tilde y_{t+1}$ and norm $\rho_{t+1}^Y$, and the prediction based on $-A^\top x_t$ then produces $y_t$.

\begin{lemma}[One-step inequality with the updated prediction]
\label{lem:alt-reused-utility}
Consider one $k$-action IR-PRM+ learner. After it plays $x_t$, observes $u_t$, and updates its regret vector, suppose that the next prediction is $m_{t+1}=u_t$. If $x_{t+1}$ is the strategy produced by this prediction, then
\begin{equation}
\ip{u_t}{x_{t+1}-x_t}
\ge
-(\rho_{t+1}-\rho_t).
\label{eq:alt-reused-utility}
\end{equation}
\end{lemma}

\begin{proof}
By \eqref{eq:rho-lower-one}, $\rho_t>0$. Write
$q_t=[r_t]_+$ and $Q_t=\norm{q_t}_1$, so
$x_t=q_t/Q_t$ and $\norm{q_t}_2=\rho_t$. Let $q_{t+1}$ be the
positive-part vector after applying the prediction $m_{t+1}=u_t$ to the
updated regret vector $\tilde r_{t+1}$. Thus, for the prediction shift
$\gamma_{t+1}$,
\[
q_{t+1}
=
[\tilde r_{t+1}+u_t-\gamma_{t+1}\1]_+,
\qquad
\norm{q_{t+1}}_2=\rho_{t+1}.
\]
This prediction preserves Euclidean norm, which gives the second
equality.

We first compare the prediction shift $\gamma_{t+1}$ with the payoff of the
previously played strategy $x_t$. Since $[v]_+\ge v$ coordinatewise and
$q_t\ge0$,
\[
\ip{q_t}{q_{t+1}}
\ge
\ip{q_t}{\tilde r_{t+1}+u_t-\gamma_{t+1}\1}.
\]
Using $\ip{q_t}{u_t}=Q_t\ip{u_t}{x_t}$ and
$\ip{q_t}{\1}=Q_t$, rearrangement gives
\begin{equation}
Q_t\bigl(\gamma_{t+1}-\ip{u_t}{x_t}\bigr)
\ge
\ip{q_t}{\tilde r_{t+1}}-\ip{q_t}{q_{t+1}}.
\label{eq:alt-reused-utility-shift-start}
\end{equation}
The first inner product on the right is at least $\rho_t^2$ by the same
argument as \eqref{eq:inner-lower}, since the update vector satisfies
$\ip{g_t}{q_t}=0$. The second is at most
$\rho_t\rho_{t+1}$ by Cauchy--Schwarz. Therefore
\[
Q_t\bigl(\gamma_{t+1}-\ip{u_t}{x_t}\bigr)
\ge
\rho_t^2-\rho_t\rho_{t+1}
=
-\rho_t(\rho_{t+1}-\rho_t).
\]
Because $Q_t=\norm{q_t}_1\ge\norm{q_t}_2=\rho_t$, division by $Q_t$
yields
\begin{equation}
\gamma_{t+1}-\ip{u_t}{x_t}
\ge
-(\rho_{t+1}-\rho_t).
\label{eq:alt-reused-utility-shift-bound}
\end{equation}

Next we show that the strategy produced by the updated prediction, $x_{t+1}$, has payoff at
least $\gamma_{t+1}$. Define
\[
h:=q_{t+1}-(\tilde r_{t+1}+u_t-\gamma_{t+1}\1)\ge0.
\]
The residual $h$ is nonzero only where the unclipped vector is negative,
whereas $q_{t+1}$ is zero there. Thus $h$ and $q_{t+1}$ have disjoint
support and $\ip{h}{x_{t+1}}=0$. With
$Q_{t+1}:=\norm{q_{t+1}}_1$ and $x_{t+1}=q_{t+1}/Q_{t+1}$, dotting
$q_{t+1}=\tilde r_{t+1}+u_t-\gamma_{t+1}\1+h$ with
$q_{t+1}$ gives
\begin{align*}
\ip{u_t}{x_{t+1}}-\gamma_{t+1}
&=
\frac{\norm{q_{t+1}}_2^2-
\ip{\tilde r_{t+1}}{q_{t+1}}}{Q_{t+1}}\\
&\ge0.
\end{align*}
The last inequality follows from Cauchy--Schwarz and
$\norm{\tilde r_{t+1}}_2=\norm{q_{t+1}}_2=\rho_{t+1}$.
Adding this nonnegative quantity to \eqref{eq:alt-reused-utility-shift-bound} gives
\[
\ip{u_t}{x_{t+1}-x_t}
=
\bigl(\ip{u_t}{x_{t+1}}-\gamma_{t+1}\bigr)
+
\bigl(\gamma_{t+1}-\ip{u_t}{x_t}\bigr)
\ge
-(\rho_{t+1}-\rho_t),
\]
which proves \eqref{eq:alt-reused-utility}.

\end{proof}

We next record the two staggered norm-growth inequalities, which use the same constants as Lemma~\ref{lem:norm-growth}; no additional game parameter is required.

\begin{lemma}[Alternating regret-vector norm growth]
\label{lem:alt-matrix-growth}
For every $t\ge2$, the alternating updates satisfy
\begin{align}
\rho_{t+1}^Y-\rho_t^Y
&\le
\frac{2(1+\sqrt{k^X})^2(1+\sqrt{k^Y})^2\norm A_2^2}
{\rho_t^Y\rho_t^X}
\notag\\[-0.3em]
&\quad\times
\left[
\frac{\rho_t^X}{2k^X}\norm{x_t-\tilde x_t}_2^2
+(\rho_t^X-\rho_{t-1}^X)
\right],
\label{eq:alt-growth-y}\\
\rho_{t+1}^X-\rho_t^X
&\le
\frac{2(1+\sqrt{k^Y})^2(1+\sqrt{k^X})^2\norm A_2^2}
{\rho_t^X\rho_{t+1}^Y}
\notag\\[-0.3em]
&\quad\times
\left[
\frac{\rho_{t+1}^Y}{2k^Y}\norm{y_t-\tilde y_{t+1}}_2^2
+(\rho_{t+1}^Y-\rho_t^Y)
\right].
\label{eq:alt-growth-x}
\end{align}
\end{lemma}

\begin{proof}
The derivation is the same two-step conversion as in
Lemma~\ref{lem:norm-growth}, but the column player is half a step ahead of
the row player. We therefore derive the two recurrences separately so that
the shifted indices are visible.

For the column player, the utility observed at its current half-step is
$-A^\top x_t$, while its previous-utility prediction was
$-A^\top x_{t-1}$. Thus its prediction error is
$-A^\top(x_t-x_{t-1})$. Applying the centering bound
\eqref{eq:center-bound}, the operator-norm inequality, and then
\eqref{eq:growth} gives
\begin{align*}
\rho_{t+1}^Y-\rho_t^Y
&\le
\frac{\norm{g_t^Y}_2^2}{2\rho_t^Y}\\
&\le
\frac{(1+\sqrt{k^Y})^2\norm A_2^2}{2\rho_t^Y}
\norm{x_t-x_{t-1}}_2^2.
\end{align*}
The row player's consecutive movement is controlled by
Corollary~\ref{cor:movement}:
\[
\norm{x_t-x_{t-1}}_2^2
\le
2\norm{x_t-\tilde x_t}_2^2
+
\frac{4(1+\sqrt{k^X})^2}{\rho_t^X}
(\rho_t^X-\rho_{t-1}^X).
\]
Substitution gives
\begin{align*}
\rho_{t+1}^Y-\rho_t^Y
&\le
\frac{(1+\sqrt{k^Y})^2\norm A_2^2}{\rho_t^Y}
\norm{x_t-\tilde x_t}_2^2\\
&\quad+
\frac{2(1+\sqrt{k^X})^2(1+\sqrt{k^Y})^2\norm A_2^2}
{\rho_t^Y\rho_t^X}
(\rho_t^X-\rho_{t-1}^X).
\end{align*}
As in the simultaneous proof, $(1+\sqrt{k^X})^2/k^X\ge1$, so
\[
\frac{(1+\sqrt{k^Y})^2\norm A_2^2}{\rho_t^Y}
\le
\frac{2(1+\sqrt{k^X})^2(1+\sqrt{k^Y})^2\norm A_2^2}
{\rho_t^Y\rho_t^X}
\frac{\rho_t^X}{2k^X}.
\]
Putting the two terms under this common coefficient yields
\eqref{eq:alt-growth-y}.

For the row player, the prediction that produced $x_t$ used $Ay_{t-1}$,
whereas the utility observed after the column update is $Ay_t$. Hence the
row prediction error is $A(y_t-y_{t-1})$, and the same argument gives
\[
\rho_{t+1}^X-\rho_t^X
\le
\frac{(1+\sqrt{k^X})^2\norm A_2^2}{2\rho_t^X}
\norm{y_t-y_{t-1}}_2^2.
\]
The movement $y_{t-1}\to y_t$ is one complete update of the column player
under this half-step indexing. Therefore Corollary~\ref{cor:movement} applies
with the newly produced norm $\rho_{t+1}^Y$:
\[
\norm{y_t-y_{t-1}}_2^2
\le
2\norm{y_t-\tilde y_{t+1}}_2^2
+
\frac{4(1+\sqrt{k^Y})^2}{\rho_{t+1}^Y}
(\rho_{t+1}^Y-\rho_t^Y).
\]
Substituting and using $(1+\sqrt{k^Y})^2/k^Y\ge1$ gives
\eqref{eq:alt-growth-x}.

The different column denominators in the two recurrences are therefore not
an extra assumption: they record which column norm exists at the relevant
point under the staggered indexing. The column recurrence uses $\rho_t^Y$ before
the column update, while the row recurrence uses
$\rho_{t+1}^Y$ after that update. All denominators are bounded below by one under the initialization \eqref{eq:positive-init}.
\end{proof}

\subsection{Proof of Theorem~\ref{thm:alt-matrix}}
\begin{proof}
Use the constant $\kappa_A$ defined in Lemma~\ref{lem:norm-growth}, namely
$\kappa_A=2(1+\sqrt{k^X})^2(1+\sqrt{k^Y})^2\norm A_2^2$.
The proof again splits according to whether the norm product is below or above the threshold, but the two
players now use staggered indices. By \eqref{eq:positive-init} and
monotonicity, both norms are at least one throughout. Suppose first that
\begin{equation}
\rho_t^X\rho_t^Y<4\kappa_A
\label{eq:alt-threshold}
\end{equation}
for every $t$. Since both norms are at least one, this product bound directly gives
\[
\rho_t^X\le4\kappa_A,
\qquad
\rho_t^Y\le4\kappa_A,
\]
so both norms are already bounded over all rounds.

Now suppose the threshold is reached. Let $\tau\ge2$ be the first iteration
from which it holds; if it already holds at iteration $1$, take $\tau=2$ and
absorb the first iteration into the constants. As in the
simultaneous proof, the first crossing can overshoot by only a bounded
amount: positive-part projection is nonexpansive, so
$\rho_{t+1}-\rho_t\le\norm{g_t}_2$, and matrix-game prediction errors are
bounded over all rounds. We therefore absorb the finite crossing and all earlier
rounds into constants and analyze the suffix $t\ge\tau$.

For $t\ge\tau$, both coefficients in
\eqref{eq:alt-growth-y}--\eqref{eq:alt-growth-x} are at most $1/4$.
For the row recurrence, the denominator contains $\rho_{t+1}^Y$, and
monotonicity gives $\rho_{t+1}^Y\ge\rho_t^Y$, so the same threshold is
sufficient. Hence
\begin{align}
\rho_{t+1}^Y-\rho_t^Y
&\le
\frac14\left[
\frac{\rho_t^X}{2k^X}\norm{x_t-\tilde x_t}_2^2
+(\rho_t^X-\rho_{t-1}^X)
\right],
\label{eq:alt-quarter-y}\\
\rho_{t+1}^X-\rho_t^X
&\le
\frac14\left[
\frac{\rho_{t+1}^Y}{2k^Y}\norm{y_t-\tilde y_{t+1}}_2^2
+(\rho_{t+1}^Y-\rho_t^Y)
\right].
\label{eq:alt-quarter-x}
\end{align}

We now sum these two staggered recurrences carefully. Define
\begin{align*}
P_T&:=\sum_{t=\tau}^T
\frac{\rho_t^X}{2k^X}\norm{x_t-\tilde x_t}_2^2,\\
Q_T&:=\sum_{t=\tau}^T
\frac{\rho_{t+1}^Y}{2k^Y}\norm{y_t-\tilde y_{t+1}}_2^2,\\
X_T&:=\rho_{T+1}^X-\rho_\tau^X,
\qquad
Y_T:=\rho_{T+1}^Y-\rho_\tau^Y,\\
b_\tau&:=\rho_\tau^X-\rho_{\tau-1}^X.
\end{align*}
Summing \eqref{eq:alt-quarter-y} gives
\begin{equation}
Y_T\le\frac14P_T+\frac14X_T+\frac14b_\tau,
\label{eq:alt-summed-y}
\end{equation}
where the extra $b_\tau$ comes from the one-index shift in the row norm
increments. Summing \eqref{eq:alt-quarter-x} gives
\begin{equation}
X_T\le\frac14Q_T+\frac14Y_T.
\label{eq:alt-summed-x}
\end{equation}
Substituting \eqref{eq:alt-summed-y} into \eqref{eq:alt-summed-x} gives
\[
X_T
\le
\frac14Q_T+\frac1{16}P_T+\frac1{16}X_T+\frac1{16}b_\tau.
\]
Thus
\[
X_T\le\frac4{15}Q_T+\frac1{15}P_T+\frac1{15}b_\tau.
\]
Putting this bound back into \eqref{eq:alt-summed-y} gives
\[
Y_T\le\frac4{15}P_T+\frac1{15}Q_T+\frac4{15}b_\tau.
\]
Adding the last two displays yields the clean combined bound
\begin{equation}
X_T+Y_T
\le
\frac13(P_T+Q_T+b_\tau).
\label{eq:alt-combined-growth}
\end{equation}
This is where the factor $1/3$ in the alternating proof comes from.

Reindex the shifted column sum in $Q_T$ by $s=t+1$. All interior terms
become the corresponding shifted prediction displacement terms
$\frac{\rho_s^Y}{2k^Y}\norm{y_{s-1}-\tilde y_s}_2^2$; only the terminal
term at $s=T+1$ remains outside the usual sum. Absorbing the finite prefix
and $b_\tau$ into a constant $C_{\rm alt}$ gives
\begin{align}
\rho_{T+1}^X+\rho_{T+1}^Y
&\le C_{\rm alt}
+
\frac13\sum_{t=1}^T\left[
\frac{\rho_t^X}{2k^X}\norm{x_t-\tilde x_t}_2^2
+
\frac{\rho_t^Y}{2k^Y}\norm{y_{t-1}-\tilde y_t}_2^2
\right]
\notag\\
&\quad+
\frac13\frac{\rho_{T+1}^Y}{2k^Y}
\norm{y_T-\tilde y_{T+1}}_2^2.
\label{eq:alt-norm-sum}
\end{align}

We next bound the prediction displacement sum on the right. Define the column player's external regret under the
half-step indexing by
\[
\Reg_Y^{\rm half}(T)
:=
\max_{y'\in\Delta_{k^Y}}\sum_{t=1}^T
\ip{-A^\top x_t}{y'-y_{t-1}}.
\]
To see the relation with the duality gap, expand the two regrets. The row
regret is
\[
\Reg_X(T)
=
T\max_{x'\in\Delta_{k^X}}\ip{A\bar y_T}{x'}
-
\sum_{t=1}^T x_t^\top Ay_t,
\]
while the half-step column external regret is
\[
\Reg_Y^{\rm half}(T)
=
\sum_{t=1}^T x_t^\top Ay_{t-1}
-
T\min_{y'\in\Delta_{k^Y}}\ip{Ay'}{\bar x_T}.
\]
Adding them and inserting and subtracting
$\sum_t x_t^\top Ay_t$ gives
\begin{equation}
\Reg_X(T)+\Reg_Y^{\rm half}(T)
=
T\Gap(\bar x_T,\bar y_T)
+
\sum_{t=1}^T\ip{-A^\top x_t}{y_t-y_{t-1}}.
\label{eq:alt-social}
\end{equation}
Thus the only difference from the usual zero-sum identity is the payoff
change caused by the column update.

Lemma~\ref{lem:alt-reused-utility}, applied to the column player under the half-step indexing, bounds each such payoff change from below:
\[
\ip{-A^\top x_t}{y_t-y_{t-1}}
\ge
-(\rho_{t+1}^Y-\rho_t^Y).
\]
Summing telescopes, and $\rho_1^Y=1$ gives a payoff-shift lower bound
$-(\rho_{T+1}^Y-\rho_1^Y)\ge-\rho_{T+1}^Y$. Since the gap in
\eqref{eq:alt-social} is nonnegative,
\begin{equation}
\Reg_X(T)+\Reg_Y^{\rm half}(T)
\ge-\rho_{T+1}^Y.
\label{eq:alt-regret-lower}
\end{equation}

Fact~\ref{fact:rvu-intermediate}, applied to the row and column updates at their respective indices, gives an upper bound on the same regret sum:
\begin{align*}
\Reg_X(T)+\Reg_Y^{\rm half}(T)
&\le
\rho_{T+1}^X+\rho_{T+1}^Y\\
&\quad-
\sum_{t=1}^T
\frac{\rho_t^X}{2k^X}\norm{x_t-\tilde x_t}_2^2\\
&\quad-
\sum_{t=1}^T
\frac{\rho_t^Y}{2k^Y}\norm{y_{t-1}-\tilde y_t}_2^2.
\end{align*}
Let
\[
H_T^{\rm alt}
:=
\sum_{t=1}^T\left[
\frac{\rho_t^X}{2k^X}\norm{x_t-\tilde x_t}_2^2
+
\frac{\rho_t^Y}{2k^Y}\norm{y_{t-1}-\tilde y_t}_2^2
\right].
\]
Combining the upper bound above with \eqref{eq:alt-regret-lower} gives
\[
-\rho_{T+1}^Y
\le
\rho_{T+1}^X+\rho_{T+1}^Y-H_T^{\rm alt}.
\]
Moving the prediction displacement terms to the left and the lower-bound term to the right gives
\begin{equation}
H_T^{\rm alt}
\le
\rho_{T+1}^X+2\rho_{T+1}^Y.
\label{eq:alt-stability}
\end{equation}

The remaining terminal prediction displacement in \eqref{eq:alt-norm-sum} is also
controlled by terminal norm. If $k^Y=1$, the column simplex is a singleton and
the term is zero. If $k^Y\ge2$, the squared simplex diameter is at most $2$,
so
\[
\frac{\rho_{T+1}^Y}{2k^Y}
\norm{y_T-\tilde y_{T+1}}_2^2
\le
\frac{\rho_{T+1}^Y}{k^Y}
\le
\frac12\rho_{T+1}^Y
\le
\frac12(\rho_{T+1}^X+\rho_{T+1}^Y).
\]
Let $S_{T+1}:=\rho_{T+1}^X+\rho_{T+1}^Y$. From
\eqref{eq:alt-stability}, the prediction displacement contribution in
\eqref{eq:alt-norm-sum} contributes at most
$\frac13(\rho_{T+1}^X+2\rho_{T+1}^Y)\le\frac23S_{T+1}$, while the
terminal prediction displacement contributes at most $\frac16S_{T+1}$. Hence
\[
S_{T+1}
\le
C_{\rm alt}+\left(\frac23+\frac16\right)S_{T+1}
=
C_{\rm alt}+\frac56S_{T+1}.
\]
Moving the last term to the left gives $S_{T+1}\le6C_{\rm alt}$, so both
norm sequences remain bounded over time.

Fact~\ref{fact:rvu-intermediate} now gives
$\Reg_X(T)=\mathcal O_T(1)$ and $\Reg_Y^{\rm half}(T)=\mathcal O_T(1)$. To relate the half-step external regret to the full-profile external regret, note that the difference
between the two comparator objectives is independent of the comparator:
\begin{align*}
\Reg_Y(T)
&=
\Reg_Y^{\rm half}(T)
-
\sum_{t=1}^T\ip{-A^\top x_t}{y_t-y_{t-1}}\\
&\le
\Reg_Y^{\rm half}(T)+\rho_{T+1}^Y
=\mathcal O_T(1),
\end{align*}
where the inequality again uses Lemma~\ref{lem:alt-reused-utility} and telescoping.
Thus both full-profile external regrets are $\mathcal O_T(1)$. The usual zero-sum
regret-to-gap identity then gives
$\Gap(\bar x_T,\bar y_T)=\mathcal O_T(T^{-1})$.
\end{proof}

\subsection{Proof of Theorem~\ref{thm:alt-matrix-current}}
\begin{proof}
The previous proof gives regret-vector norms that remain bounded over time and the
prediction displacement bound \eqref{eq:alt-stability}. Since every player starts with norm one, its norm
is bounded below by a fixed positive constant. Therefore
\eqref{eq:alt-stability} implies square summability of the row prediction
displacements $\norm{x_t-\tilde x_t}_2$ and the shifted column prediction
displacements $\norm{y_{t-1}-\tilde y_t}_2$.

Lemma~\ref{lem:update-change} gives the corresponding post-update displacement
bounds. Specifically,
\[
\norm{\tilde x_{t+1}-x_t}_2^2
\le C_X(\rho_{t+1}^X-\rho_t^X),
\]
and, under the half-step column indexing,
\[
\norm{\tilde y_{t+1}-y_{t-1}}_2^2
\le C_Y(\rho_{t+1}^Y-\rho_t^Y),
\]
for fixed finite constants $C_X,C_Y$. The right-hand sides are summable
because the bounded monotone norm increments telescope. Inserting these
updated strategies between consecutive played strategies therefore gives
\[
\sum_t\norm{x_t-x_{t-1}}_2^2<\infty,
\qquad
\sum_t\norm{y_t-y_{t-1}}_2^2<\infty.
\]

We now bound the two current best-response residuals. The column strategy $y_t$ is produced using $-A^\top x_t$ as its
prediction, which is also the current column utility.
Lemma~\ref{lem:prediction-residual} therefore gives directly
\begin{equation}
\max_{y'\in\Delta_{k^Y}}\ip{-A^\top x_t}{y'-y_t}
\le
2\sqrt{2k^Y}\,\rho_{t+1}^Y
\norm{y_t-\tilde y_{t+1}}_2.
\label{eq:alt-current-y}
\end{equation}
For $t\ge2$, the row strategy $x_t$ was produced from the prediction
$Ay_{t-1}$, whereas its current utility is $Ay_t$. Adding and subtracting
$Ay_{t-1}$ gives, for every $x'\in\Delta_{k^X}$,
\begin{align*}
\ip{Ay_t}{x'-x_t}
&=
\ip{Ay_{t-1}}{x'-x_t}
+
\ip{A(y_t-y_{t-1})}{x'-x_t}\\
&\le
2\sqrt{2k^X}\,\rho_t^X\norm{x_t-\tilde x_t}_2
+
\sqrt2\norm A_2\norm{y_t-y_{t-1}}_2.
\end{align*}
Maximizing over $x'$ yields
\begin{align}
\max_{x'\in\Delta_{k^X}}\ip{Ay_t}{x'-x_t}
&\le
2\sqrt{2k^X}\,\rho_t^X
\norm{x_t-\tilde x_t}_2
\notag\\[-0.2em]
&\quad+
\sqrt2\norm A_2\norm{y_t-y_{t-1}}_2.
\label{eq:alt-current-x}
\end{align}
The factor $\sqrt2$ in the second term is the Euclidean diameter bound
$\norm{x'-x_t}_2\le\sqrt2$.

The norm sequences are bounded, and every prediction displacement and consecutive
strategy movement on the right sides of \eqref{eq:alt-current-y}--\eqref{eq:alt-current-x} is square summable.
Thus both current best-response residuals are square summable. The row display is used for
$t\ge2$; its iteration-$1$ residual is finite and contributes only one bounded
term. Since the matrix-game gap is the sum of the two residuals,
\[
\sum_{t=1}^{\infty}\Gap(x_t,y_t)^2<\infty.
\]
Consequently $\Gap(x_t,y_t)\to0$.

Let $K:=\sum_t\Gap(x_t,y_t)^2<\infty$. Then
\[
T\min_{1\le t\le T}\Gap(x_t,y_t)^2\le K,
\]
so the best-current-iterate gap is $\mathcal O_T(T^{-1/2})$. Finally, if the distance
to $\mathcal Z^\star$ did not tend to zero, a subsequence would stay a
fixed positive distance away. Compactness would give a convergent
subsubsequence, and continuity of the gap would make its limit a zero-gap
point, hence a Nash equilibrium, a contradiction. Therefore
$\operatorname{dist}((x_t,y_t),\mathcal Z^\star)\to0$.
\end{proof}

\section{Proofs for Alternating IR-PCFR+}
\label{app:alt-cfr}
\subsection{Proof of Theorem~\ref{thm:alt-cfr}}
\begin{proof}
Let
\[
h_t=(x_t,y_{t-1}),
\qquad
z_t=(x_t,y_t).
\]
The proof has the same structure as the alternating matrix-game proof, but
we must track finitely many information sets. We first make the local
norm-growth recurrences explicit, then sum them with a threshold
argument, and finally use the full-reach equilibrium weights to control the
prediction displacement terms.

For each information set $H$, let $k_H:=|A(H)|$. For $I\in\mathcal I_X$ and $J\in\mathcal I_Y$, let $L_I^X$ and $L_J^Y$ denote finite Euclidean Lipschitz constants for the respective counterfactual-utility maps $F_I^X$ and $F_J^Y$ on the product of behavioral simplices.

For $t\ge2$, the local learner at column-player information set $J$ compares the new half-step profile
$h_t$ with $h_{t-1}$, while the local learner at row-player information set $I$ compares the new full
profile $z_t$ with $z_{t-1}$. Therefore
\eqref{eq:center-bound}, Lipschitzness, and \eqref{eq:growth} give
\begin{align}
\rho^Y_{J,t+1}-\rho^Y_{J,t}
&\le
\frac{(1+\sqrt{k_J})^2(L_J^Y)^2}{2\rho^Y_{J,t}}
\norm{h_t-h_{t-1}}_2^2,
\label{eq:alt-cfr-y-growth}\\
\rho^X_{I,t+1}-\rho^X_{I,t}
&\le
\frac{(1+\sqrt{k_I})^2(L_I^X)^2}{2\rho^X_{I,t}}
\norm{z_t-z_{t-1}}_2^2.
\label{eq:alt-cfr-x-growth}
\end{align}
We now expand the two profile movements over information sets:
\begin{align*}
\norm{h_t-h_{t-1}}_2^2
&=
\sum_{K\in\mathcal I_X}\norm{x_{K,t}-x_{K,t-1}}_2^2
+
\sum_{L\in\mathcal I_Y}\norm{y_{L,t-1}-y_{L,t-2}}_2^2,\\
\norm{z_t-z_{t-1}}_2^2
&=
\sum_{K\in\mathcal I_X}\norm{x_{K,t}-x_{K,t-1}}_2^2
+
\sum_{L\in\mathcal I_Y}\norm{y_{L,t}-y_{L,t-1}}_2^2.
\end{align*}
The row-player local utilities are evaluated under full-profile indexing, while the two column
movements differ by one half-step.

Define the nonnegative quantities
\begin{align*}
B^X_{K,t}
&:=
\frac{\rho^X_{K,t}}{2k_K}\norm{x_{K,t}-\tilde x_{K,t}}_2^2
+(\rho^X_{K,t}-\rho^X_{K,t-1}),\\
B^Y_{L,t}
&:=
\frac{\rho^Y_{L,t}}{2k_L}\norm{y_{L,t-1}-\tilde y_{L,t}}_2^2
+(\rho^Y_{L,t}-\rho^Y_{L,t-1}).
\end{align*}
Corollary~\ref{cor:movement}, with the same algebra as in
\eqref{eq:product-block-movement}, gives
\begin{align*}
\norm{x_{K,t}-x_{K,t-1}}_2^2
&\le
\frac{4(1+\sqrt{k_K})^2}{\rho^X_{K,t}}B^X_{K,t},\\
\norm{y_{L,t-1}-y_{L,t-2}}_2^2
&\le
\frac{4(1+\sqrt{k_L})^2}{\rho^Y_{L,t}}B^Y_{L,t},\\
\norm{y_{L,t}-y_{L,t-1}}_2^2
&\le
\frac{4(1+\sqrt{k_L})^2}{\rho^Y_{L,t+1}}B^Y_{L,t+1}.
\end{align*}

Substituting these information-set bounds into
\eqref{eq:alt-cfr-y-growth}--\eqref{eq:alt-cfr-x-growth} exposes the same
two reciprocal norms as in the simultaneous product proof. To write the
coefficients compactly, for any local information set
$H\in\mathcal I_X\cup\mathcal I_Y$ define
\begin{align*}
c^{Y}_{J,H}
&:=2(1+\sqrt{k_J})^2(L_J^Y)^2(1+\sqrt{k_H})^2,\\
c^{X}_{I,H}
&:=2(1+\sqrt{k_I})^2(L_I^X)^2(1+\sqrt{k_H})^2.
\end{align*}
Then the column recurrence is
\begin{align}
\rho^Y_{J,t+1}-\rho^Y_{J,t}
&\le
\sum_{K\in\mathcal I_X}
\frac{c^Y_{J,K}}{\rho^Y_{J,t}\rho^X_{K,t}}B^X_{K,t}
+
\sum_{L\in\mathcal I_Y}
\frac{c^Y_{J,L}}{\rho^Y_{J,t}\rho^Y_{L,t}}B^Y_{L,t}.
\label{eq:alt-cfr-y-recurrence}
\end{align}
% \eqref{eq:alt-cfr-y-growth}--\eqref{eq:alt-cfr-x-growth} exposes the same two reciprocal norms as in the simultaneous product proof. To write the coefficients compactly, for any local information set
% $H\in\mathcal I_X\cup\mathcal I_Y$ define
% \begin{align*}
% c^{Y}_{J,H}
% &:=2(1+\sqrt{k_J})^2(L_J^Y)^2(1+\sqrt{k_H})^2,\\
% c^{X}_{I,H}
% &:=2(1+\sqrt{k_I})^2(L_I^X)^2(1+\sqrt{k_H})^2.
% \end{align*}

% Then the column recurrence is
% \begin{align}
% \rho^Y_{J,t+1}-\rho^Y_{J,t}
% &\le
% \sum_{K\in\mathcal I_X}
% \frac{c^Y_{J,K}}{\rho^Y_{J,t}\rho^X_{K,t}}B^X_{K,t}
% +
% \sum_{L\in\mathcal I_Y}
% \frac{c^Y_{J,L}}{\rho^Y_{J,t}\rho^Y_{L,t}}B^Y_{L,t},
% \label{eq:alt-cfr-y-recurrence}
% \end{align}
and the row recurrence is
\begin{align}
\rho^X_{I,t+1}-\rho^X_{I,t}
&\le
\sum_{K\in\mathcal I_X}
\frac{c^X_{I,K}}{\rho^X_{I,t}\rho^X_{K,t}}B^X_{K,t}
+
\sum_{L\in\mathcal I_Y}
\frac{c^X_{I,L}}{\rho^X_{I,t}\rho^Y_{L,t+1}}B^Y_{L,t+1}.
\label{eq:alt-cfr-x-recurrence}
\end{align}
These displays are the local analogue of
\eqref{eq:alt-growth-y}--\eqref{eq:alt-growth-x}. The only difference
between the two indexings is the index $t+1$ on the column quantities entering the
row recurrence.

We next sum these recurrences. Let $M$ be the total number of information sets
and fix $\varepsilon\in(0,1/3)$. Every local norm is at least one. Hence, exactly as
in Theorem~\ref{thm:weighted-minty}, we can choose a finite threshold for each local learner
such that, whenever its regret-vector norm reaches that threshold, every
coefficient in its recurrence
\eqref{eq:alt-cfr-y-recurrence} or
\eqref{eq:alt-cfr-x-recurrence} is at most $\varepsilon/M$. Bounded
counterfactual utilities give a round-independent one-step bound on local norm growth,
so the first crossing of each threshold has only a finite overshoot. Put all
below-threshold growth and all crossing overshoots into a finite constant
$\widetilde C_\varepsilon$.

Let
\[
S_T:=
\sum_{I\in\mathcal I_X}\rho^X_{I,T+1}
+
\sum_{J\in\mathcal I_Y}\rho^Y_{J,T+1}
\]
and
\[
\mathcal S_T:=
\sum_{t,I}\frac{\rho^X_{I,t}}{2k_I}
\norm{x_{I,t}-\tilde x_{I,t}}_2^2
+
\sum_{t,J}\frac{\rho^Y_{J,t}}{2k_J}
\norm{y_{J,t-1}-\tilde y_{J,t}}_2^2.
\]
At a fixed round, summing the column recurrences above threshold makes the
total coefficient on each $B^X_{K,t}$ and $B^Y_{L,t}$
at most $\varepsilon$. Summing the above-threshold row recurrences likewise
makes the total coefficient on each $B^X_{K,t}$ and shifted
$B^Y_{L,t+1}$ at most $\varepsilon$. Consequently
\begin{align*}
S_T
\le
\widetilde C_\varepsilon
+\varepsilon\sum_t\left[
2\sum_K B^X_{K,t}
+
\sum_L B^Y_{L,t}
+
\sum_L B^Y_{L,t+1}
\right].
\end{align*}
The row quantities appear twice because they enter both the half-step and full-profile
movements. The column quantities appear once under each of the two shifted indexings.

Expanding these quantities, the stability parts contribute at most
$2\varepsilon\mathcal S_T$, apart from the single terminal column stability
term created by reindexing $B^Y_{L,t+1}$. The norm-increment parts telescope:
the row increments appear twice and the column increments appear twice.
The terminal column prediction displacement term is at most the corresponding terminal
norm because
\[
\frac{\rho^Y_{L,T+1}}{2k_L}
\norm{y_{L,T}-\tilde y_{L,T+1}}_2^2
\le
\rho^Y_{L,T+1},
\]
using the squared simplex diameter bound $2$. Thus all telescoping and
endpoint terms together are at most $3\varepsilon S_T$, after absorbing
finite initial boundary terms into $\widetilde C_\varepsilon$. We obtain
\[
S_T
\le
\widetilde C_\varepsilon
+2\varepsilon\mathcal S_T
+3\varepsilon S_T.
\]
Since $\varepsilon<1/3$, moving the last term to the left gives
\begin{equation}
S_T\le
\frac{\widetilde C_\varepsilon}{1-3\varepsilon}
+
\frac{2\varepsilon}{1-3\varepsilon}\,\mathcal S_T.
\label{eq:alt-cfr-high-low}
\end{equation}
Thus
\eqref{eq:alt-cfr-high-low} gives the required relation between terminal
norm growth and the accumulated prediction displacement terms. The denominator
$1-3\varepsilon$ and numerator $2\varepsilon$ arise from the staggered
prediction displacement and terminal-norm contributions above.

We now control $\mathcal S_T$ using the equilibrium weights supplied by the
full-reach condition. First, apply the pointwise CFR decomposition at
$h_t=(x_t,y_{t-1})$ with $y_t$ as the column comparator. Since the
comparator induces $q_t^Y$, this gives
\[
U(q_t^X,q_{t-1}^Y)-U(q_t^X,q_t^Y)
=
\sum_J q_t^Y[p_J]
\ip{F_J^Y(h_t)}{y_{J,t}-y_{J,t-1}}.
\]

Lemma~\ref{lem:alt-reused-utility} gives, for each column information set $J$,
\[
\ip{F_J^Y(h_t)}{y_{J,t}-y_{J,t-1}}
\ge
-(\rho^Y_{J,t+1}-\rho^Y_{J,t}).
\]
Since $0\le q_t^Y[p_J]\le1$ and every norm increment is nonnegative,
weighting and summing preserves the lower bound
\begin{equation}
U(q_t^X,q_{t-1}^Y)-U(q_t^X,q_t^Y)
\ge
-\sum_J(\rho^Y_{J,t+1}-\rho^Y_{J,t}).
\label{eq:alt-cfr-payoff-lower}
\end{equation}

Next fix a full-reach equilibrium
$(q^{X,\star},q^{Y,\star})$. Applying the pointwise CFR decomposition with
the equilibrium strategies as comparators gives the row deviation term at
$z_t$ and the column deviation term at $h_t$. Their sum is
% Next take a full-reach equilibrium
% $(q^{X,\star},q^{Y,\star})$. The pointwise CFR decomposition gives the
% row equilibrium-deviation term at $z_t$ and the column equilibrium-deviation
% term at $h_t$. Their sum is
\begin{align*}
&\sum_Iq^{X,\star}[p_I]
\ip{F_I^X(z_t)}{x_I^\star-x_{I,t}}\\
&\quad+
\sum_Jq^{Y,\star}[p_J]
\ip{F_J^Y(h_t)}{y_J^\star-y_{J,t-1}}\\
&=
\bigl[U(q^{X,\star},q_t^Y)-U(q_t^X,q_t^Y)\bigr]\\
&\quad+
\bigl[U(q_t^X,q_{t-1}^Y)-U(q_t^X,q^{Y,\star})\bigr].
\end{align*}
Because $(q^{X,\star},q^{Y,\star})$ is a saddle point,
$U(q^{X,\star},q_t^Y)\ge U(q_t^X,q^{Y,\star})$. Dropping their
nonnegative difference yields
\begin{align}
&\sum_Iq^{X,\star}[p_I]
\ip{F_I^X(z_t)}{x_I^\star-x_{I,t}}
\notag\\
&\quad+
\sum_Jq^{Y,\star}[p_J]
\ip{F_J^Y(h_t)}{y_J^\star-y_{J,t-1}}
\notag\\
&\quad\ge
U(q_t^X,q_{t-1}^Y)-U(q_t^X,q_t^Y).
\label{eq:alt-cfr-staggered-minty}
\end{align}
To compare the lower and upper estimates on exactly the same quantity, define
\begin{align*}
W_T
:=\sum_{t=1}^T\Bigg[&
\sum_Iq^{X,\star}[p_I]
\ip{F_I^X(z_t)}{x_I^\star-x_{I,t}}\\
&+\sum_Jq^{Y,\star}[p_J]
\ip{F_J^Y(h_t)}{y_J^\star-y_{J,t-1}}
\Bigg].
\end{align*}
Combining \eqref{eq:alt-cfr-staggered-minty} with
\eqref{eq:alt-cfr-payoff-lower} and summing over $t$ gives
\begin{align*}
W_T
&\ge -\sum_{t=1}^T\sum_J
      (\rho^Y_{J,t+1}-\rho^Y_{J,t})\\
&= -\sum_J(\rho^Y_{J,T+1}-\rho^Y_{J,1})\\
&\ge -\sum_J\rho^Y_{J,T+1}.
\end{align*}
The second line is the telescoping step; the last inequality uses
$\rho^Y_{J,1}\ge0$.

On the other hand, each fixed-equilibrium local comparator sum appearing in
$W_T$ is at most the corresponding local external regret. Applying
Fact~\ref{fact:rvu-intermediate} to those row external regrets and half-step
column external regrets gives the upper bound
\begin{align*}
W_T
&\le
\sum_Iq^{X,\star}[p_I]\rho^X_{I,T+1}\\
&\quad+
\sum_Jq^{Y,\star}[p_J]\rho^Y_{J,T+1}\\
&\quad-
\sum_{t,I}q^{X,\star}[p_I]
\frac{\rho^X_{I,t}}{2k_I}
\norm{x_{I,t}-\tilde x_{I,t}}_2^2\\
&\quad-
\sum_{t,J}q^{Y,\star}[p_J]
\frac{\rho^Y_{J,t}}{2k_J}
\norm{y_{J,t-1}-\tilde y_{J,t}}_2^2.
\end{align*}
Putting the lower and upper bounds on the same quantity together and moving
the prediction displacement terms to the left gives
\begin{align}
&\sum_{t,I}q^{X,\star}[p_I]
\frac{\rho^X_{I,t}}{2k_I}
\norm{x_{I,t}-\tilde x_{I,t}}_2^2
+
\sum_{t,J}q^{Y,\star}[p_J]
\frac{\rho^Y_{J,t}}{2k_J}
\norm{y_{J,t-1}-\tilde y_{J,t}}_2^2
\notag\\
&\quad\le
\sum_Iq^{X,\star}[p_I]\rho^X_{I,T+1}
+
\sum_Jq^{Y,\star}[p_J]\rho^Y_{J,T+1}
+
\sum_J\rho^Y_{J,T+1}.
\label{eq:alt-cfr-stability-weighted}
\end{align}
Under the full-reach condition, every parent-sequence weight is strictly positive. Let
$w_{\min}>0$ and $w_{\max}$ be their minimum and maximum. The left side of
\eqref{eq:alt-cfr-stability-weighted} is at least
$w_{\min}\mathcal S_T$, while the right side is at most
$(w_{\max}+1)S_T$. Hence
\[
\mathcal S_T\le K_wS_T,
\qquad
K_w:=\frac{w_{\max}+1}{w_{\min}}.
\]

Substitute this into \eqref{eq:alt-cfr-high-low}:
\[
S_T
\le
\widetilde C_\varepsilon
+
\frac{2\varepsilon K_w}{1-3\varepsilon}S_T.
\]
Choose $\varepsilon<1/(3+2K_w)$. This inequality is exactly the condition
$2\varepsilon K_w<1-3\varepsilon$, so the coefficient of $S_T$ on the
right is strictly less than one. Moving it to the left proves that $S_T$ is bounded independently of $T$.

Dropping the negative prediction displacement terms in
Fact~\ref{fact:rvu-intermediate} now gives $\mathcal O_T(1)$ local external regret at all row
information sets and $\mathcal O_T(1)$ half-step local external regret at all column information sets. The
standard CFR decomposition transfers the row bounds directly to full-profile
sequence-form external regret.

For the column player, the half-step and full-profile sequence-form external
regrets differ only by the realized payoff shift
\[
D_T:=\sum_{t=1}^T
\bigl[U(q_t^X,q_{t-1}^Y)-U(q_t^X,q_t^Y)\bigr].
\]
Summing \eqref{eq:alt-cfr-payoff-lower} gives
\[
D_T\ge-\sum_J\rho^Y_{J,T+1}\ge-C_Y,
\]
where $C_Y<\infty$ is independent of $T$, since the preceding bound on
$S_T$ bounds the finite sum of column regret-vector norms. Thus
$-D_T\le C_Y$. For every fixed column comparator, replacing $q_{t-1}^Y$ by $q_t^Y$ in the
realized profile subtracts exactly $D_T$ from the cumulative regret
expression. Because this correction is independent of the comparator, the
same identity holds after maximizing over comparators. Hence the
full-profile column sequence-form external regret is at most the half-step
sequence-form external regret plus $C_Y$, and is therefore
$\mathcal O_T(1)$. Both players thus have $\mathcal O_T(1)$ full-profile sequence-form
external regret. The exact sequence-form zero-sum regret-to-gap identity then yields
$\Gap_{\mathcal Q}(\bar q_T^X,\bar q_T^Y)=\mathcal O_T(T^{-1})$.
\end{proof}

\subsection{Proof of Theorem~\ref{thm:alt-cfr-current}}
\begin{proof}
The preceding proof gives a bound on $S_T$ independent of $T$ and the weighted
stability estimate \eqref{eq:alt-cfr-stability-weighted}. Because the positive parent-sequence weights are bounded below by $w_{\min}>0$, the unweighted
stability sum $\mathcal S_T$ is finite over the infinite horizon. Every local norm is at least one, so every local prediction displacement appearing in $\mathcal S_T$ is square summable.

Lemma~\ref{lem:update-change} converts bounded monotone local norm growth
into square summability of the corresponding post-update displacements, exactly
as in the simultaneous product proof. Since there are finitely many information sets, summing these bounds over
information sets preserves finiteness. In particular,
\begin{equation}
\sum_t\norm{z_t-z_{t-1}}_2^2<\infty,
\qquad
\sum_t\norm{z_t-h_t}_2^2<\infty.
\label{eq:alt-cfr-z-movement}
\end{equation}
The first sum records full-profile movement. The second records only the
column movement, since $z_t-h_t=(0,y_t-y_{t-1})$.

We next bound the local current residuals. For $t\ge2$, consider first a row information
set $I$. Then $x_{I,t}$ was generated using the previous
full-profile prediction $F_I^X(z_{t-1})$, while the current utility is
$F_I^X(z_t)$. Adding and subtracting that prediction, applying
Lemma~\ref{lem:prediction-residual} to the predicted term, and applying
Lipschitzness plus the simplex diameter bound to the utility difference gives
\begin{align}
R^X_{I,t}
&\le
2\sqrt{2k_I}\,\rho^X_{I,t}
\norm{x_{I,t}-\tilde x_{I,t}}_2
+
\sqrt2L_I^X\norm{z_t-z_{t-1}}_2.
\label{eq:alt-cfr-current-x}
\end{align}
For a column information set $J$, $y_{J,t}$ is
generated from the half-step prediction $F_J^Y(h_t)$. Its current
full-profile utility is $F_J^Y(z_t)$. The same add-and-subtract argument
therefore gives
\begin{align}
R^Y_{J,t}
&\le
2\sqrt{2k_J}\,\rho^Y_{J,t+1}
\norm{y_{J,t}-\tilde y_{J,t+1}}_2
+
\sqrt2L_J^Y\norm{z_t-h_t}_2.
\label{eq:alt-cfr-current-y}
\end{align}
The first term in each display is a residual term for the predicted utility; the second is
the Lipschitz error between the profile used in the prediction and the
current full profile.

All norm sequences in the displays remain bounded over time, and every prediction displacement
and profile movement term there is square summable by the preceding paragraph and
\eqref{eq:alt-cfr-z-movement}. Thus every local residual is square summable; the finitely many initial boundary residuals are finite.

Let $R_t$ be the sum of all row and column local residuals. Because there are
finitely many information sets, Cauchy--Schwarz gives
$\sum_tR_t^2<\infty$. The pointwise CFR decomposition, exactly as in the
proof of Theorem~\ref{thm:cfr-current}, gives
\[
0\le\Gap_{\mathcal Q}(q_t^X,q_t^Y)\le R_t.
\]
Therefore
\[
\sum_{t=1}^{\infty}\Gap_{\mathcal Q}(q_t^X,q_t^Y)^2<\infty,
\]
which proves \eqref{eq:alt-cfr-current-square-sum}. In particular the
current gap tends to zero.

If
$K:=\sum_t\Gap_{\mathcal Q}(q_t^X,q_t^Y)^2<\infty$, then
\[
T\min_{1\le t\le T}\Gap_{\mathcal Q}(q_t^X,q_t^Y)^2\le K,
\]
so the best-current-iterate gap is $\mathcal O_T(T^{-1/2})$. Finally, if the distance
to $\mathcal Z_{\mathcal Q}^\star$ did not tend to zero, compactness of the
realization-plan polytopes would give a convergent subsequence staying a
fixed positive distance away. Continuity of the gap would make its limit a
zero-gap point, hence a sequence-form Nash equilibrium, a contradiction.
Thus
$\operatorname{dist}((q_t^X,q_t^Y),\mathcal Z_{\mathcal Q}^\star)\to0$.
\end{proof}

\section{Standard Zero Initialization}
\label{app:zero-init}
The main text uses $\tilde r_1=k^{-1/2}\1$, so that $\rho_1=1$, only to
avoid repeated division-by-zero cases. Here we prove the same conclusions
for the standard initialization $\tilde r_1=0$. We do not repeat the
analysis for $\rho_t>0$; we treat only rounds with $\rho_t=0$ and the first
transition to positive norm.

For zero initialization, fix an arbitrary fixed default strategy
$x^o\in\Delta_k$. Whenever $\rho_t=0$, set
\[
\tilde x_t=x_t=x^o,
\qquad
m_t=0,
\]
and after observing $u_t$ update
\begin{equation}
\tilde r_{t+1}
=
\pos{u_t-\ip{u_t}{x^o}\1}.
\label{eq:zero-init-update}
\end{equation}
If this vector is zero, use the same rule on the next round. Once the stored regret vector has positive norm, use the standard IR-PRM+
update and the previous-utility predictor from the main text.

\subsection{Handling zero regret vectors}
The following observations are the only additional ingredients needed for
zero initialization.

\begin{lemma}[Zero-initialization boundary cases]
\label{lem:zero-boundary}
Consider a zero-initialized $k$-action learner with the rule above.
Then:
\begin{enumerate}
\item If $\rho_t=\rho_{t+1}=0$, the played strategy does not move and the
current utility has zero best-response residual at $x^o$.
\item If $\rho_t=0<\rho_{t+1}$, the post-update displacement bound
\eqref{eq:update-change} still holds on this first positive transition.
\item If $\tau:=\min\{t:\rho_t>0\}<\infty$ and
$\beta:=\rho_\tau$, then $\beta>0$ and
$\rho_t\ge\beta$ for every $t\ge\tau$.
\item If $\rho_t=0$, the residual bound for the predicted utility
\eqref{eq:prediction-residual} holds trivially because $m_t=0$.
\end{enumerate}
\end{lemma}

\begin{proof}
If \eqref{eq:zero-init-update} is zero, then
$u_{t,a}\le\ip{u_t}{x^o}$ for every action $a$. Since
$\ip{u_t}{x^o}$ is a convex combination of the coordinates of $u_t$, this
forces
\[
\max_a u_{t,a}=\ip{u_t}{x^o},
\]
so the best-response residual at the fixed default strategy is zero. The rule
also keeps $x_{t+1}=x_t=x^o$ as long as $\tilde r_t=0$.

For the first positive transition, the right side of
\eqref{eq:update-change} becomes
$2(1+\sqrt{k})^2$, because $\rho_t=0$. The squared Euclidean diameter of
$\Delta_k$ is at most $2$, so the same inequality covers that transition.
After $\rho_t$ becomes positive, the monotonicity argument in the proof of
Lemma~\ref{lem:update-change} applies without change, giving
$\rho_t\ge\rho_\tau=\beta>0$ thereafter. Finally, $m_t=0$ whenever $\rho_t=0$,
so the left side of \eqref{eq:prediction-residual} is zero.
\end{proof}

The main proofs used the numerical lower bound $\rho_t\ge1$ only to keep
reciprocal norm factors uniformly bounded before the threshold is reached and to remove the
weights $\rho_t$ from square-summability statements. On any suffix on which
$\rho_t\ge\beta>0$, the same arguments hold with $1$ replaced by $\beta$;
only the constants change. Thus the finitely many initial rounds with $\rho_t=0$ can be separated
from the argument for $\rho_t>0$.

\subsection{Simultaneous matrix games}
Consider first the two-player simultaneous dynamics. If one player's regret vector
remains zero forever, its strategy is fixed. After one observation, the
other player's previous-utility prediction is exact against that fixed
strategy, so $u_t-m_t=0$ and hence $g_t=0$ and its regret-vector norm stops
growing. Hence both regrets are bounded directly.

Otherwise both regret-vector norms become positive. Let $\tau_X,\tau_Y$ be their first
positive times and let
\[
\beta_X:=\rho^X_{\tau_X}>0,
\qquad
\beta_Y:=\rho^Y_{\tau_Y}>0.
\]
After $\tau=\max\{\tau_X,\tau_Y\}$, monotonicity gives
$\rho_t^X\ge\beta_X$ and $\rho_t^Y\ge\beta_Y$. The coupled recurrences
\eqref{eq:normX}--\eqref{eq:normY} therefore apply on this suffix. Before the product threshold is reached, the fixed lower bounds $\beta_X,\beta_Y$ replace the
unit lower bounds used in the proof of Theorem~\ref{thm:main}; after the
threshold is crossed, the proof is identical. The finitely many preceding
rounds contribute only a finite additive constant. This proves the same
$\mathcal O_T(1)$ regret and $\mathcal O_T(T^{-1})$ average-gap guarantees.

For current-profile convergence, Lemma~\ref{lem:zero-boundary} handles the
rounds with zero norm and the first transitions to positive norm. After $\tau$, the factors
$\rho_t^X$ and $\rho_t^Y$ in the stability sums can be removed using
$\beta_X$ and $\beta_Y$ instead of one. The remainder of the proof of
Theorem~\ref{thm:current-matrix} is unchanged. Hence the square-summable
current gap, convergence in distance to the Nash equilibrium set, and the $\mathcal O_T(T^{-1/2})$
best-current-iterate bound all continue to hold.

\subsection{Products of simplices and simultaneous IR-PCFR+}
For the product-space argument, let $\mathcal A$ be the set of
learners whose regret-vector norms ever become positive. A learner outside $\mathcal A$
maintains a fixed strategy, and the first part of Lemma~\ref{lem:zero-boundary} shows that
its instantaneous local regret and current local residual are identically
zero. In particular, its comparator term in the weighted Minty sum is
nonpositive, so removing those terms can only increase that sum; the
weighted Minty lower bound therefore remains valid on $\mathcal A$. For
each $j\in\mathcal A$, let $\beta_j>0$ be its first positive norm. Since there are finitely many learners, after the last transition to positive norm, every $j\in\mathcal A$ satisfies
$\rho_{j,t}\ge\beta_j$.

On that suffix, the proof of Lemma~\ref{lem:product-growth} applies with the
sum restricted to $j\in\mathcal A$. In the proof of
Theorem~\ref{thm:weighted-minty}, choose the threshold $\Theta_i$ so that
\[
\frac{c_{ij}}{\Theta_i\beta_j}
\le
\frac{w_{\min}}{4Nw_{\max}}
\qquad(j\in\mathcal A).
\]
This is exactly the threshold argument from the main proof with the unit
lower bound replaced by the fixed positive values $\beta_j$. Each learner has at most one transition from zero to positive norm, so all initial
zero-norm rounds and threshold crossings contribute only a finite constant. Indices outside
$\mathcal A$ contribute neither strategy movement nor regret. The weighted Minty
summation then closes exactly as before. The current-residual argument is
unchanged after replacing the unit lower bounds by the $\beta_j$ values.
Consequently Theorem~\ref{thm:weighted-minty} and
Corollary~\ref{cor:weighted-current} also hold under zero initialization.
Applying the same CFR decomposition proves the zero-initialized versions of
Theorems~\ref{thm:irpcfr} and~\ref{thm:cfr-current}.

\subsection{Alternating updates}
For alternating updates, one additional boundary calculation is needed
because the second player uses the newly observed utility as its prediction.
The proof of Lemma~\ref{lem:alt-reused-utility} for positive regret-vector norm already covers every later iteration. We verify the first positive transition here.

Suppose $\rho_t=0<\rho_{t+1}$. Write
$c=\ip{u_t}{x^o}$ and
\[
p=[u_t-c\1]_+,
\qquad
z=\frac{p}{\norm p_1}.
\]
On every coordinate where $p_a>0$, $u_{t,a}-c=p_a$, so
\[
\ip{u_t}{z}-c
=
\frac{\norm p_2^2}{\norm p_1}>0.
\]
Now apply the prediction based on $u_t$. Write
\[
r=[p+u_t-\gamma\1]_+,
\qquad
x^+=\frac{r}{\norm r_1},
\qquad
h=r-(p+u_t-\gamma\1)\ge0.
\]
Norm preservation gives $\norm r_2=\norm p_2$, and $h$ has disjoint
support from $r$. Using $u_t-\gamma\1=r-p-h$ gives
\[
\ip{u_t}{x^+}-\gamma
=
\frac{\norm r_2^2-\ip{p}{r}}{\norm r_1},
\]
whereas dotting the same identity with $z=p/\norm p_1$ gives
\[
\ip{u_t}{z}-\gamma
=
\frac{\ip{r}{p}-\norm p_2^2}{\norm p_1}
-
\ip{h}{z}.
\]
Subtracting and using $\norm r_2=\norm p_2$ yields
\begin{align*}
\ip{u_t}{x^+-z}
&=
\bigl(\norm p_2^2-\ip{p}{r}\bigr)
\left(\frac1{\norm r_1}+\frac1{\norm p_1}\right)
+\ip{h}{z}\\
&=
\frac{\norm{r-p}_2^2}{2}
\left(\frac1{\norm r_1}+\frac1{\norm p_1}\right)
+\ip{h}{z}
\ge0.
\end{align*}
Thus this prediction does not reduce payoff relative to $z$, and
$\ip{u_t}{x_{t+1}-x_t}\ge0$. This is stronger than
\eqref{eq:alt-reused-utility} on the first positive transition. If $\rho_{t+1}=0$, the claim is immediate because the played strategy remains $x^o$.

Once a matrix-game player's or a CFR local learner's regret-vector norm first becomes positive, it has a fixed positive lower bound and all subsequent
staggered recurrences are exactly those proved in the main alternating
analysis. If such a regret vector never attains positive norm, its strategy does not move and its
residual for the utility at which the update with $\rho_t=0$ is performed is
zero. In alternating CFR, when the local regret vector at a column-player information set $J$ is zero, the residual at the full
profile differs from its half-step residual only through the utility
change from $h_t$ to $z_t$, and is bounded by
$\sqrt2L_J^Y\norm{z_t-h_t}_2$; the latter movement is already square
summable in the alternating proof. Since there are only finitely many matrix-game players and CFR local learners, for those that eventually become positive the prefixes before their first positive norm, together with the first positive transitions themselves, contribute only finitely many bounded terms; those that remain at zero are handled above. The proofs of
Theorems~\ref{thm:alt-matrix}, \ref{thm:alt-matrix-current},
\ref{thm:alt-cfr}, and~\ref{thm:alt-cfr-current} therefore carry over with
constants depending on the first positive regret-vector norms.

\begin{corollary}[Zero initialization]
\label{cor:zero-init}
All regret, average-gap, current-gap, Nash-set convergence, and
best-current-iterate conclusions proved in this paper remain valid when
IR-PRM+ and IR-PCFR+ are initialized with $\tilde r_1=0$ and use the
rule above for $\tilde r_t=0$.
\end{corollary}

\section{Normalization Under $\ell_1$ Scaling}
\label{app:normalization}

\begin{lemma}[Normalization bound]
\label{lem:normalization}
If $u,v\in\R_+^k$ satisfy $\norm u_2=\norm v_2=1$, then
\begin{equation}
\norm{\frac{u}{\norm u_1}-\frac{v}{\norm v_1}}_2
\le
(1+\sqrt{k})\norm{u-v}_2.
\label{eq:normalization-simple}
\end{equation}
\end{lemma}

\begin{proof}
Let $a=\norm u_1$ and $b=\norm v_1$. Adding and subtracting $v/a$ gives
\begin{align*}
\norm{\frac ua-\frac vb}_2
&\le
\norm{\frac ua-\frac va}_2
+
\norm{\frac va-\frac vb}_2\\
&=
\frac{\norm{u-v}_2}{a}
+
\frac{|a-b|}{ab}\norm v_2.
\end{align*}
Because $u$ and $v$ are nonnegative Euclidean unit vectors,
$a,b\ge1$ and $\norm v_2=1\le b$. Therefore
\[
\norm{\frac ua-\frac vb}_2
\le
\norm{u-v}_2+|a-b|.
\]
By the reverse triangle inequality and norm comparison,
\[
|a-b|
=
|\norm u_1-\norm v_1|
\le
\norm{u-v}_1
\le
\sqrt{k}\norm{u-v}_2.
\]
Combining the last two displays proves \eqref{eq:normalization-simple}.
\end{proof}

\begin{remark}[Sharper normalization constant]
The factor $1+\sqrt{k}$ in Lemma~\ref{lem:normalization} can be improved
to $\sqrt{k}$ by a more involved analysis of the $\ell_1$ normalization
map. This sharper constant changes only the dimension-dependent constants
in our bounds and does not affect any qualitative conclusion or asymptotic
rate. We therefore use the elementary bound above throughout.
\end{remark}

\section{A Two-Displacement RVU Bound for IR-PRM+}
\label{app:two-displacement-rvu}

The single-learner analysis of prediction and post-update displacements also yields a strengthened RVU
inequality for IR-PRM+. The existing RVU bound contains a negative term
for the prediction displacement $x_t-\tilde x_t$. Lemma~\ref{lem:update-change}
shows that growth of the regret-vector norm controls the post-update
displacement $\tilde x_{t+1}-x_t$, allowing an explicit stability term for
this second part of the update as well. This result is not needed for the
main convergence analysis, which works directly with regret-vector norm
growth, but it gives a two-displacement RVU form that makes both the prediction
and post-update displacements explicit.

For comparison, Theorem~3.2 of Zhang et al.~\cite{zhang2026} has the form
\begin{equation}
\Reg(T)
\le
\left(
\rho_1^2+
\sum_{t=1}^T\norm{g_t}_2^2
\right)^{1/2}
-
\frac{1}{2k}
\sum_{t=1}^T
\rho_t\norm{x_t-\tilde x_t}_2^2.
\label{eq:zhang-rvu-form}
\end{equation}

\begin{theorem}[Two-displacement IR-PRM+ RVU bound]
\label{thm:two-displacement-rvu}
For a $k$-action IR-PRM+ learner with any nonnegative initial regret vector,
\begin{align}
\Reg(T)
\le{}&
2\left(
\rho_1^2+
\sum_{t=1}^T\norm{g_t}_2^2
\right)^{1/2}
-\rho_1
\notag\\
&-
\frac{1}{2k}
\sum_{t=1}^T
\rho_t\norm{x_t-\tilde x_t}_2^2
\notag\\
&-
\frac{1}{2(1+\sqrt{k})^2}
\sum_{t=1}^T
\rho_{t+1}\norm{\tilde x_{t+1}-x_t}_2^2.
\label{eq:two-displacement-rvu}
\end{align}
\end{theorem}

\begin{proof}
Suppose first that $\rho_t>0$. Lemma~\ref{lem:update-change} gives
\[
\norm{\tilde x_{t+1}-x_t}_2^2
\le
\frac{2(1+\sqrt{k})^2(\rho_{t+1}-\rho_t)}
{\rho_{t+1}}.
\]
Multiplying by $\rho_{t+1}/[2(1+\sqrt{k})^2]$ yields
\begin{equation}
\frac{\rho_{t+1}}{2(1+\sqrt{k})^2}
\norm{\tilde x_{t+1}-x_t}_2^2
\le
\rho_{t+1}-\rho_t.
\label{eq:update-displacement-one-step}
\end{equation}

The same inequality holds when $\rho_t=0$. If $\rho_{t+1}=0$, the rule
in Appendix~\ref{app:zero-init} leaves both $x_t$ and
$\tilde x_{t+1}$ equal to the fixed strategy $x^o$, so the left-hand
side of \eqref{eq:update-displacement-one-step} is zero. If
$\rho_t=0<\rho_{t+1}$, then
$x_t,\tilde x_{t+1}\in\Delta_k$, and hence
\[
\norm{\tilde x_{t+1}-x_t}_2^2\le2.
\]
Therefore
\[
\frac{\rho_{t+1}}{2(1+\sqrt{k})^2}
\norm{\tilde x_{t+1}-x_t}_2^2
\le
\frac{\rho_{t+1}}{(1+\sqrt{k})^2}
\le
\rho_{t+1}
=
\rho_{t+1}-\rho_t.
\]
Thus \eqref{eq:update-displacement-one-step} holds for every round.
Summing over $t$ gives
\begin{equation}
\sum_{t=1}^T
\frac{\rho_{t+1}}{2(1+\sqrt{k})^2}
\norm{\tilde x_{t+1}-x_t}_2^2
\le
\rho_{T+1}-\rho_1.
\label{eq:update-displacement-summed}
\end{equation}

The squared-norm increment bound for IR-PRM$^+$ established by Zhang et al.~\cite{zhang2026} gives
\[
\rho_{t+1}^2-\rho_t^2
\le
\norm{g_t}_2^2,
\]
including when $\rho_t=0$. Summing yields
\begin{equation}
\rho_{T+1}
\le
\left(
\rho_1^2+
\sum_{t=1}^T\norm{g_t}_2^2
\right)^{1/2}.
\label{eq:rho-terminal-variation}
\end{equation}
Combining \eqref{eq:update-displacement-summed} and
\eqref{eq:rho-terminal-variation} gives
\begin{align}
&\sum_{t=1}^T
\frac{\rho_{t+1}}{2(1+\sqrt{k})^2}
\norm{\tilde x_{t+1}-x_t}_2^2
\notag\\
&\qquad\le
\left(
\rho_1^2+
\sum_{t=1}^T\norm{g_t}_2^2
\right)^{1/2}
-\rho_1.
\label{eq:update-displacement-variation}
\end{align}
Equivalently,
\begin{align}
&\left(
\rho_1^2+
\sum_{t=1}^T\norm{g_t}_2^2
\right)^{1/2}
-\rho_1
\notag\\
&\qquad-
\sum_{t=1}^T
\frac{\rho_{t+1}}{2(1+\sqrt{k})^2}
\norm{\tilde x_{t+1}-x_t}_2^2
\ge0.
\label{eq:update-displacement-nonnegative}
\end{align}
Adding this nonnegative quantity to the right-hand side of
\eqref{eq:zhang-rvu-form} gives \eqref{eq:two-displacement-rvu}.
\end{proof}

\section*{AI-Assisted Technology Disclosure}

OpenAI ChatGPT 5 and 6 was used during the development of this work in extended multi-turn interactions for research brainstorming, including generating and refining candidate hypotheses, exploring possible proof strategies and counterexamples, identifying technical obstacles, and assisting with the presentation and organization of the resulting arguments. It was also used for language editing, LaTeX editing, and checks for notational and expository consistency.

For the theoretical results, AI-generated suggestions were treated as candidate arguments rather than as authoritative results. The authors independently selected the research directions pursued, wrote and checked the mathematical arguments presented in the paper, and take responsibility for all statements, proofs, citations, and conclusions.

The interactions consisted of iterative prompts seeking to analyze regret-matching and predictive regret-matching dynamics, derive candidate inequalities and convergence arguments, search for counterexamples to proposed claims, compare alternative proof approaches, and refine the resulting theory and exposition. Because the research process involved extended conversations across multiple sessions, not every individual prompt and model version was separately retained. The authors therefore report the nature of the interactions here rather than reconstructing a lengthy and in some cases unavailable sequence of prompts.

\bibliographystyle{ACM-Reference-Format}
\bibliography{references}

@inproceedings{syrgkanis2015,
  author    = {Vasilis Syrgkanis and Alekh Agarwal and Haipeng Luo and Robert E. Schapire},
  title     = {Fast Convergence of Regularized Learning in Games},
  booktitle = {Advances in Neural Information Processing Systems},
  volume    = {28},
  pages     = {2989--2997},
  year      = {2015}
}

@inproceedings{tammelin2015,
  author    = {Oskari Tammelin and Neil Burch and Michael Johanson and Michael Bowling},
  title     = {Solving Heads-Up Limit {Texas Hold'em}},
  booktitle = {Proceedings of the 24th International Joint Conference on Artificial Intelligence},
  pages     = {645--652},
  year      = {2015}
}

@article{farina2021predictive,
  author  = {Gabriele Farina and Christian Kroer and Tuomas Sandholm},
  title   = {Faster Game Solving via Predictive Blackwell Approachability: Connecting Regret Matching and Mirror Descent},
  journal = {Proceedings of the AAAI Conference on Artificial Intelligence},
  volume  = {35},
  number  = {6},
  pages   = {5363--5371},
  year    = {2021},
  doi     = {10.1609/aaai.v35i6.16676}
}

@inproceedings{farina2023instability,
  author    = {Gabriele Farina and Julien Grand-Cl{\'e}ment and Christian Kroer and Chung-Wei Lee and Haipeng Luo},
  title     = {Regret Matching$^+$: (In)Stability and Fast Convergence in Games},
  booktitle = {Advances in Neural Information Processing Systems},
  volume    = {36},
  pages     = {61546--61572},
  year      = {2023}
}

@inproceedings{cai2025last,
  author    = {Yang Cai and Gabriele Farina and Julien Grand-Cl{\'e}ment and Christian Kroer and Chung-Wei Lee and Haipeng Luo and Weiqiang Zheng},
  title     = {Last-Iterate Convergence Properties of Regret-Matching Algorithms in Games},
  booktitle = {International Conference on Learning Representations},
  year      = {2025}
}

@inproceedings{cai2025reduction,
  author    = {Yang Cai and Haipeng Luo and Chen-Yu Wei and Weiqiang Zheng},
  title     = {From Average-Iterate to Last-Iterate Convergence in Games: A Reduction and Its Applications},
  booktitle = {Advances in Neural Information Processing Systems},
  volume    = {38},
  pages     = {52398--52428},
  year      = {2025},
  doi       = {10.52202/085713-1563}
}

@inproceedings{meng2026faster,
  author    = {Linjian Meng and Youzhi Zhang and Shangdong Yang and Wenbin Li and Tianyu Ding and Yang Gao},
  title     = {A Faster Parameter-Free Regret Matching Algorithm},
  booktitle = {International Conference on Learning Representations},
  year      = {2026}
}

@inproceedings{lee2021last,
  author    = {Chung-Wei Lee and Christian Kroer and Haipeng Luo},
  title     = {Last-Iterate Convergence in Extensive-Form Games},
  booktitle = {Advances in Neural Information Processing Systems},
  volume    = {34},
  pages     = {14293--14305},
  year      = {2021}
}

@inproceedings{chakrabarti2024treeplex,
  author    = {Darshan Chakrabarti and Julien Grand-Cl{\'e}ment and Christian Kroer},
  title     = {Extensive-Form Game Solving via Blackwell Approachability on Treeplexes},
  booktitle = {Advances in Neural Information Processing Systems},
  volume    = {37},
  pages     = {35257--35287},
  year      = {2024},
  doi       = {10.52202/079017-1111}
}

@inproceedings{meng2025efg,
  author    = {Linjian Meng and Tianpei Yang and Youzhi Zhang and Zhenxing Ge and Shangdong Yang and Tianyu Ding and Wenbin Li and Bo An and Yang Gao},
  title     = {Efficient Last-Iterate Convergence in Solving Extensive-Form Games},
  booktitle = {Advances in Neural Information Processing Systems},
  volume    = {38},
  pages     = {111586--111633},
  year      = {2025},
  doi       = {10.52202/085713-3367}
}

@inproceedings{zhang2026,
  author    = {Brian Hu Zhang and Ioannis Anagnostides and Tuomas Sandholm},
  title     = {Scale-Invariant Regret Matching and Online Learning with Optimal Convergence: Bridging Theory and Practice in Zero-Sum Games},
  booktitle = {Proceedings of the 27th ACM Conference on Economics and Computation},
  year      = {2026}
}

@inproceedings{zinkevich2007,
  author    = {Martin Zinkevich and Michael Johanson and Michael Bowling and Carmelo Piccione},
  title     = {Regret Minimization in Games with Incomplete Information},
  booktitle = {Advances in Neural Information Processing Systems},
  volume    = {20},
  pages     = {1729--1736},
  year      = {2007}
}

@inproceedings{farina2019,
  author    = {Gabriele Farina and Christian Kroer and Tuomas Sandholm},
  title     = {Regret Circuits: Composability of Regret Minimizers},
  booktitle = {Proceedings of the 36th International Conference on Machine Learning},
  series    = {Proceedings of Machine Learning Research},
  volume    = {97},
  pages     = {1863--1872},
  publisher = {PMLR},
  year      = {2019}
}

@inproceedings{koller1994,
  author    = {Daphne Koller and Nimrod Megiddo and Bernhard von Stengel},
  title     = {Fast Algorithms for Finding Randomized Strategies in Game Trees},
  booktitle = {Proceedings of the 26th Annual ACM Symposium on Theory of Computing},
  pages     = {750--759},
  publisher = {ACM},
  year      = {1994},
  doi       = {10.1145/195058.195451}
}

@article{vonstengel1996,
  author  = {Bernhard von Stengel},
  title   = {Efficient Computation of Behavior Strategies},
  journal = {Games and Economic Behavior},
  volume  = {14},
  number  = {2},
  pages   = {220--246},
  year    = {1996},
  doi     = {10.1006/game.1996.0050}
}

@article{nemirovski2004,
  author  = {Arkadi Nemirovski},
  title   = {Prox-Method with Rate of Convergence $O(1/t)$ for Variational
             Inequalities with Lipschitz Continuous Monotone Operators and
             Smooth Convex-Concave Saddle Point Problems},
  journal = {SIAM Journal on Optimization},
  volume  = {15},
  number  = {1},
  pages   = {229--251},
  year    = {2004},
  doi     = {10.1137/S1052623403425629}
}

@inproceedings{meng2025smooth,
  author    = {Linjian Meng and Youzhi Zhang and Zhenxing Ge and Tianyu Ding and Shangdong Yang and Zheng Xu and Wenbin Li and Yang Gao},
  title     = {Last-Iterate Convergence of Smooth Regret Matching+ Variants in Learning Nash Equilibria},
  booktitle = {Advances in Neural Information Processing Systems},
  volume    = {38},
  pages     = {21747--21792},
  year      = {2025},
  doi       = {10.52202/085713-0644}
}

@article{nesterov2005,
  author  = {Nesterov, Yurii},
  title   = {Excessive Gap Technique in Nonsmooth Convex Minimization},
  journal = {SIAM Journal on Optimization},
  volume  = {16},
  number  = {1},
  pages   = {235--249},
  year    = {2005},
  doi     = {10.1137/S1052623403422285}
}

@inproceedings{rakhlin2013,
  author    = {Rakhlin, Alexander and Sridharan, Karthik},
  title     = {Online Learning with Predictable Sequences},
  booktitle = {Proceedings of the 26th Annual Conference on Learning Theory},
  series    = {Proceedings of Machine Learning Research},
  volume    = {30},
  pages     = {993--1019},
  year      = {2013},
  publisher = {PMLR}
}

@inproceedings{chiang2012,
  author    = {Chiang, Chao-Kai and Yang, Tianbao and Lee, Chia-Jung
               and Mahdavi, Mehrdad and Lu, Chi-Jen and Jin, Rong
               and Zhu, Shenghuo},
  title     = {Online Optimization with Gradual Variations},
  booktitle = {Proceedings of the 25th Annual Conference on Learning Theory},
  series    = {Proceedings of Machine Learning Research},
  volume    = {23},
  pages     = {6.1--6.20},
  year      = {2012},
  publisher = {PMLR}
}

\end{document}